\RequirePackage{silence}
\documentclass{elsarticle}

\usepackage{lineno,hyperref}
\usepackage{graphicx}
\usepackage{enumitem}
\usepackage{pgf,tikz}
\usepackage[ruled,vlined]{algorithm2e} % For algorithms
\SetAlFnt{\small}
\SetAlCapFnt{\small}
\SetAlCapNameFnt{\small}
\SetAlCapHSkip{0pt}
\IncMargin{-\parindent}

\usepackage{amsmath}
\usepackage{amsthm}
\usepackage{amsfonts}
\usepackage{multirow}
\usepackage{lscape}
\usepackage{afterpage}
\usepackage{subcaption}
\usepackage{soul}
\usepackage{todonotes}
\usepackage[normalem]{ulem}
\usepackage{comment}
\usepackage{lmodern}
\usepackage{thmtools}

\renewcommand{\sout}[1]{} %%SCOMMENTARE PER FAR SCOMPARIRE IL TESTO IN MYSOUT

\allowdisplaybreaks

\newtheorem{theorem}{Theorem}

\newtheorem{corollary}{Corollary}

\newtheorem{definition}{Definition}
\newtheorem{example}{Example}

\newtheorem{obs}{Observation}

\definecolor{dartmouthgreen}{rgb}{0.05, 0.5, 0.06}

\DeclareMathOperator{\argmax}{argmax}

\newcommand{\bigtimes}{\text{\LARGE $\times$}}
\newcommand{\NPHard}{$\mathsf{NP}$-hard}
\newcommand{\NP}{$\mathsf{NP}$}
\newcommand{\Poly}{$\mathsf{P}$}
\newcommand{\APX}{$\mathsf{APX}$}

\newcommand{\MoV}{\textsf{MoV}}
\newcommand{\Expec}{\mathbb{E}}

\journal{Journal of \LaTeX\ Templates}

\begin{document}

\begin{frontmatter}

\title{Election Manipulation on Social Networks:\\ Seeding, Edge Removal, Edge Addition}
\tnotetext[mytitlenote]{A preliminary version of this work appeared in~\cite{ElectionManipulationAAAI2020}, which includes some results on edge addition/removal.}

\author{Matteo Castiglioni, Nicola Gatti, Giulia Landriani}
\address{Politecnico di Milano, Italy}

\author{Diodato Ferraioli}
\address{Universit\`a degli Studi di Salerno, Italy }

\begin{abstract}
We focus on the \emph{election manipulation problem} through \emph{social influence}, where a \emph{manipulator} exploits a \emph{social network} to make her most preferred candidate win an election.
Influence is due to information \emph{in favor of} and/or \emph{against} one or multiple candidates, sent by \emph{seeds} and spreading through the network according to the independent cascade model.
We provide a comprehensive study of the election control problem, investigating two forms of manipulations: \emph{seeding} to buy influencers given a social network, and \emph{removing} or \emph{adding edges} in the social network given the seeds and the information sent.
In particular, we study a wide range of cases distinguishing for the number of candidates or the kind of information spread over the network. 
Our main result is positive for democracy, and it shows that the election manipulation problem is \emph{not} affordable in the worst-case except for trivial classes of instances, even when one accepts to approximate the margin of victory.
In the case of seeding, we also show that the manipulation is hard even if the graph is a line and that a large class of algorithms, including most of the approaches recently adopted for social-influence problems, fail to compute a bounded approximation even on elementary networks, as undirected graphs with every node having a degree at most two or directed trees.
%
%In the case of edge addition or removal, the hardness holds even when the manipulator has an \emph{unlimited} budget, being allowed to add or remove an arbitrary number of edges.
%
%Our hardness results also apply to the basic case of social influence maximization or minimization which was not explored so far. 
%
In the case of edge removal or addition, our hardness results also apply to the basic case of social influence maximization/minimization. 
In contrast, the hardness of election manipulation holds even when the manipulator has an \emph{unlimited} budget, being allowed to remove or add an arbitrary number of edges.
%
%Our hardness results also apply to the basic case of social influence maximization or minimization which was not explored so far. 
Interestingly, our hardness results for seeding and edge removal/addition still hold in a \emph{reoptimization} variant, where the manipulator already knows an optimal solution to the problem and  computes a new solution once a local modification occurs, \emph{e.g.}, the removal/addition of  a single edge.
\end{abstract}

\end{frontmatter}

%\linenumbers

\section{Introduction}
Nowadays, social network media are the most used, if not the unique, sources of information. 
This indisputable fact turned out to influence most of our daily actions, and also to have severe effects  on the political life of our countries. 
Indeed, in many of the recent political elections around the world, there has been evidence that false or incomplete news spread through these media influenced the electoral outcome. 
For example, in the recent US presidential election, several studies show that, on average, 92\% of Americans remembered pro-Trump false news, while 23\% of them remembered the pro-Clinton fake news~\cite{allcott2017social,guess2018selective}.  
As another example, automated accounts in Twitter spread a considerable amount of political news to alter the outcome of the 2017 French elections~\cite{ferrara2017disinformation}. 
It also emerged that the fake news, spread over the major social media during the campaign for the 2018 Italian political election, is linked with the content of populist parties that won that election~\cite{alaphilippe2018disinformation,giglietto2018mapping}.

The increasing use of social networks to convey inaccurate and unverified information can lead to severe and undesired consequences, as widespread panic, libelous campaigns, and conspiracies, and it represents a menace for democracy. 
In this scenario, some natural questions are to understand to which extent the spread of (mis)information on social network media may alter the result of a political election and how to mitigate or block it. 
The former problem is known in the  literature as \emph{election control through social influence}, and it has recently been the object of interest of many works in the artificial intelligence community. 
For instance, Sina~\emph{et al.}~study a plurality voting scenario in which the voters can vote iteratively, and they show how to modify the relationship among voters to make the desired candidate to win an election~\cite{sina2015adapting}.
 Auletta~\emph{et al.} study a majority dynamics scenario and show that, in the case of two only candidates, a manipulator controlling the order in which information is disclosed to voters can lead the minority to become a majority~\cite{auletta2015minority,auletta2017information,auletta2017robustness}.
Auletta~\emph{et al.} study a similar adversary in~\cite{auletta2018reasoning}, showing that such a manipulator can lead a bare majority to consensus, but these results do not extend to the case with more than two candidates, as showed in~\cite{AulettaFFG19}.
Bredereck and Elkind study a majority dynamics scenario, showing how selecting seeds from which to diffuse information to manipulate a two-candidate election~\cite{bredereck2017manipulating}.

\subsection{Main Related Works}
Recently, Wilder and Vorobeychik studied a seeding problem in which all the seeds send the same information, either in favor of (\emph{positive}) or against (\emph{negative}) a single candidate, to make that candidate either to win or to lose, respectively, the election~\cite{wilder2018controlling}.
In particular, voters are not strategic, and ranks describe their preferences, and, given any pair of candidates $c,c'$ such that $c$ directly precedes $c'$ in the rank, positive information on $c'$ or negative information on $c$ make $c$ and $c'$ switch.
The diffusion of the information on the social network is described by the independent cascade model~\cite{kempe2003maximizing}.
The authors provide approximation algorithms for plurality voting when the objective function is the maximization of the margin of victory. 
These approximation results also hold when other voting rules and/or other diffusion models are adopted, as showed by Cor\`o \emph{et al.} in~\cite{DBLP:conf/atal/CoroCDP19, CoroCDP19}.
While the works mentioned above assume that the manipulator has complete knowledge about the problem, some recent work also deals with uncertainty on the network~\cite{abs-1905-04694}.

The works by Wilder and Vorobeychik~\cite{wilder2018controlling} and Cor\`o \emph{et al.}~\cite{DBLP:conf/atal/CoroCDP19, CoroCDP19} present some limitations when dealing with elections with more than two candidates. 
A major limitation is assuming that all the seeds send the same information, and this information is on a single candidate.~\footnote{To the best of our knowledge, the spreading of multiple information with the independent cascade model is only studied in scenarios different from election control, \emph{e.g.}, \cite{DBLP:journals/corr/abs-1906-00074}.}
Indeed, spreading simultaneously positive and/or negative information on multiple candidates can be, in some settings, necessary to make the manipulator's candidate win the election, as showed in the following example.
% \textcolor{orange}{***ma e' vero anche con 2 candidati e score arbitrari??? l'esempio sotto e' stato costruito per il caso senza score. In caso affermativo, metterei una footnote senza parlare espressamente di score.***}
%
\begin{example}
\begin{figure}
	\includegraphics[width=\linewidth]{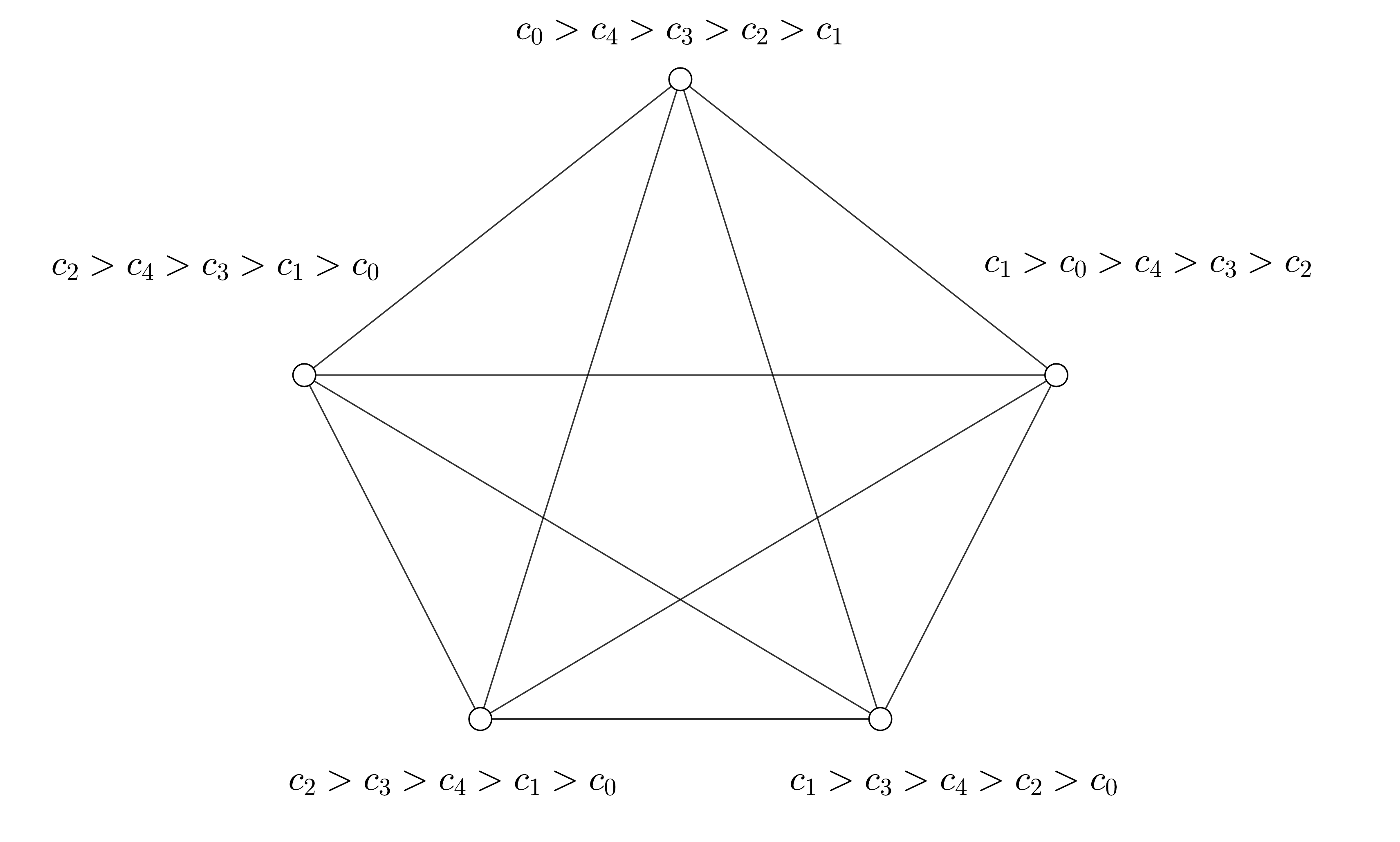}
	\caption{Clique with five voters and five candidates.}
	\label{fig:clique}
\end{figure}
Consider the setting in Figure~\ref{fig:clique}: there are five voters (corresponding to the nodes of the graph), five candidates $c_0, c_1, \ldots, c_4$, in which $c_0$ is the manipulator's candidate, and each voter receives the information spread by her neighbors with probability one. 
% Among the five voters, we assume that one of them already prefers the manipulator's candidate $c_0$ to $c_4$ and these two to the remaining candidates;
% the preference of the remaining voters are arranged as follows:
% one voter prefers $c_1$ to $c_0$ and these two to the remaining candidates;
% one voter prefers $c_1$ to $c_3$ and these two to the remaining candidates;
% one voter prefers $c_2$ to $c_3$ and these two to the remaining candidates;
% one voter prefers $c_2$ to $c_4$ and these two to the remaining candidates.
% We assume that voters are arranged on a clique and
%
A candidate gains one position in the rank thanks to positive information on her, looses one position due to negative information on her, and the manipulator has a budget sufficient for seeding two nodes, each sending information on a single candidate.
When only information on a single candidate is sent, then the desired candidate $c_0$ cannot be made to win the election (at most, the election ends with a tie between $c_0$ and another candidate, both taking two votes).
Instead, injecting the network with positive information on $c_0$ and negative information on $c_2$ results in $c_0$ being the only node with two votes, and thus the winner. 
\end{example}
% \textcolor{orange}{***NOTA: l'esempio sopra e' senza score. Ci ho pensato e io non li metterei perche' li definiamo in sezione 2.***}
%
Another major limitation is the assumption that seeding is the only action the manipulator can tackle to manipulate the election. 
However, this is not the case when the manipulator is (or collaborate with) the network media manager. In this case, the manipulator can also alter the structure of the network.
In particular, she may indefinitely conceal information exchanged among two  voters that are connected in the social network, or she may reveal information spread by unknown sources (\emph{e.g.}, as sponsored content or through mechanisms as friend suggestions). 
That is, such a manipulator can remove or add edges in the network to obstruct or push the diffusion of information.~\footnote{To the best of our knowledge, the removal or the addition of edges in the network are forms of manipulations studied only for simpler diffusion models, \emph{e.g.}, with two candidates and simple information diffusion dynamics~\cite{bredereck2017manipulating}, and when no information is spread, but voters update their votes in an iterative voting process by effect of selfish voting~\cite{sina2015adapting,AIIA}.}

\subsection{Original Contributions}
In this work, we focus on the election control problem, proposing a more general model than those available in the literature and providing a comprehensive study of the complexity of manipulating the election.

\subsubsection{Model and Motivation}

We extend the model provided in~\cite{wilder2018controlling}, along with two different directions.
First, we assume that the seeds can send different information and that the information sent by every single seed can be simultaneously positive and negative on multiple candidates.
%
%Hence, in opposition to the single news setting of \cite{wilder2018controlling,DBLP:conf/atal/CoroCDP19, CoroCDP19}, in which all messages spreading over the network are carrying the same news about a single candidate, we are now allowing different messages to carry different news and, possibly, each different news to target (positively or negatively) a different candidate.
%
%
A simple interpretation is that the seeds can share different news articles and that each news article is related to a single candidate.
We name the collection of information sent by every single seed \emph{message}.
As in~\cite{wilder2018controlling}, we model the diffusion of information according to a variant of the independent cascade model, capable of capturing the simultaneous spread of multiple different messages.

The second direction along which our model differs from the previous one is that voters' preferences are modeled by a scoring function that assigns, for every voter, a finite score to each candidate.
Every positive (negative, respectively) information on a candidate received by a voter from each seed increases (decreases, respectively) her score.
The actual score increase or decrease depends on the amount of information sent by the seed and can differ for every specific candidate. 
A simple interpretation is that a single news article can increase or decrease the score of a candidate by one and that every single seed can send multiple different news articles, each related to a specific candidate.
As a result, differently from~\cite{wilder2018controlling}, given two candidates $c,c'$ where $c$ directly precedes $c'$, a single news article in favor of $c'$ or against $c$ does not necessarily make them switch as the difference of their scores can be arbitrary.
On the other side, our model also allows a candidate to gain (lose, respectively) more than one position in the rank of a voter due to a large amount of received positive (negative, respectively) news articles on that candidate, sent by a single seed or by multiple seeds.
Thus, a voter that is \emph{uncertain} on $c_0$ can be modeled with a scoring function in which the difference between the score of the most preferred candidate and that one of $c_0$ is sufficiently small that the manipulator can change the voter's preferences to make $c_0$ be the most preferred.
Conversely, a voter that is \emph{certain} on $c_0$ can be modeled with a scoring function such that the manipulator cannot (\emph{e.g.}, due to a limitation on the news articles the manipulator can inject in the network) either make $c_0$  be the most preferred candidate if, initially, $c_0$ is not the most-preferred, or make $c_0$ not be the most-preferred if instead, initially, $c_0$ is the most preferred.

In the paper, we often refer to a special basic setting, called with \emph{single-news-article messages}, in which all the seeds send the same information, this information is only on a single candidate and induces a score increase/decrease of one. 
We also use the term \emph{unitary score distances} to refer to the case in which the difference in the score of two candidates $c,c'$ where $c$ directly precedes $c'$ is exactly one.
When our model is with single-news-article messages and unitary score distances and the number of candidates is two, it is directly comparable to that one studied in~\cite{wilder2018controlling}. 
Instead, with three or more candidates the models are not comparable, as, differently from our model, in~\cite{wilder2018controlling}, a candidate cannot increase/decrease more than one position in the rank of a voter, even if this receives multiple positive/negative messages as sent by multiple seeds.

\subsubsection{Complexity Results}
We  focus on the maximization of the increase in the margin of victory of the manipulator's candidate $c_0$, as done by Wilder and Vorobeychik in~\cite{wilder2018controlling}, and we provide a comprehensive study of the election manipulation problem when two forms of manipulations are possible:  seeding to buy influencers given a social network, and removing or adding edges in the social network given the seeds and their messages.
The manipulator is subject to budget constraints, expressing the maximum amount of cumulative news articles the seeds can spread over the network or the maximum amount of edges she can remove or add in the network.
In Table~\ref{table:summaryofresults}, we summarize our main original results.

\begin{table}[t]
\renewcommand{\arraystretch}{1.4}
	\resizebox{\columnwidth}{!}{%
		\begin{tabular}{c|c|c|c|c}
			\hline \cline{1-5}\hline\hline\cline{1-5}
			\multicolumn{5}{c}{\bf Seeding} \\ 
			\cline{1-5}
			\textsc{budget} & \multicolumn{2}{c}{\textsc{single-news-article messages}} \vline &  \multicolumn{2}{c}{\textsc{general setting}} \\ 
						& \multicolumn{2}{c}{\textsc{2 candidates}} \vline &  \textsc{2 or more candidates} & \textsc{3 or more candidates} \\	
			 &  \textsc{unitary score distances} & \textsc{arbitrary score distances} & $\delta \leq B$, $\delta$ \textsc{fixed}& $B < \delta$   \\ \hline 
% 			\textsc{budget} &2 \textsc{candidates}, $\delta = 1$& 2 \textsc{or more candidates},  &    2 \textsc{or more candidates},\\ 
% 			 &  &   $\delta \leq B$, $\delta$ \textsc{fixed}& $\delta > B$   \\ \hline 
			 limited  &   \APX ~ \cite{wilder2018controlling}& $\notin$ \APX ~ (Thm~\ref{thm:DkS}) & \APX ~ (Thm~\ref{thm:approx})& $\notin$ Exp-\APX ~ (Thm~\ref{thm:inapprox}) \\
% 			 & & (even for 2 candidates) & &\\% (even for unitary score distances)\\
			 \hline \cline{1-5}\hline\hline\cline{1-5}
%			 unlimited  &\Poly  & \Poly & \Poly \\			\hline \cline{1-4}\hline\hline\cline{1-4}
			\multicolumn{5}{c}{\bf Edge Removal} \\ \hline 
			 \textsc{budget}			& \multicolumn{2}{c}{\textsc{single-news-article messages, unitary score distances}} \vline &  \multicolumn{2}{c}{\textsc{arbitrary messages and unitary score distances}} \\  %\cline{2-4}
			 &2 \textsc{candidates}&  3 \textsc{or more candidates} &   \multicolumn{2}{c}{2 \textsc{or more candidates}} \\ \hline 
			 limited  &   $\notin$ \APX ~ (Cor~\ref{cor:EAFixed}) & $\notin$ Exp-\APX ~ (Thm~\ref{thm:EA_single_message}) & \multicolumn{2}{c}{$\notin$ Exp-\APX ~ (Thm~\ref{thm:EA_multiple_messages})}\\\cline{1-5}
			 unlimited  &\Poly ~ (Obs~\ref{obs:polyECEA})~($\dagger$)  & $\notin$ Exp-\APX ~ (Thm~\ref{thm:EA_single_message}) & \multicolumn{2}{c}{$\notin$ Exp-\APX ~ (Thm~\ref{thm:EA_multiple_messages})} \\ 			\hline \cline{1-5}\hline\hline\cline{1-5}
			\multicolumn{5}{c}{\bf Edge Addition} \\ 
			\cline{1-5}
			 \textsc{budget} & \multicolumn{2}{c|}{\textsc{single-news-article messages, unitary score distances}}  &  \multicolumn{2}{c}{\textsc{arbitrary messages, unitary score distances}} \\  %\cline{2-3}
			 & 2 \textsc{candidates}&  3 \textsc{or more candidates} &    \multicolumn{2}{c}{2 \textsc{or more candidates}}\\ \hline 
			 limited  &   $\notin$ \APX ~ (Cor~\ref{cor:ERFixed}) & $\notin$ \APX ~ (Thm~\ref{thm:ER_single_message})& \multicolumn{2}{c}{$\notin$ Exp-\APX ~ (Thm~\ref{thm:ER_multiple_messages})} \\\cline{1-5}
			 unlimited  &\Poly ~ (Obs~\ref{obs:polyECER})~($\dagger$) & $\notin$ \APX ~ (Thm~\ref{thm:ER_single_message}) & \multicolumn{2}{c}{$\notin$ Exp-\APX ~ (Thm~\ref{thm:ER_multiple_messages})} \\ \hline \cline{1-5}\hline\hline\cline{1-5}
	\end{tabular}}
	\caption{Complexity results (previously known in the literature or originally provided in this paper) on the election manipulation problem trough social influence. The case of seeding with unlimited budget is trivial, as discussed in the paper, and therefore omitted. Results marked with ($\dagger$) also hold with arbitrary score distances.}
\label{table:summaryofresults}
\end{table}

In the case of seeding, the problem is trivial when the budget available to the manipulator is unlimited, as the manipulator can make all the nodes be seeds spreading an arbitrarily large number of news articles in favor of $c_0$ and against all the other candidates. 
When instead the budget $B$ is finite, our results depend on the scoring function of voters and the budget available to the manipulator. 
Initially, we observe that the setting with single-news-article messages and two candidates becomes inapproximable within a constant factor as soon as the score distances become strictly larger than one.
In the general setting (\emph{i.e.}, when messages and score distances are arbitrary), the complexity depends on the cost $\delta$ needed to make the most reluctant voter  vote $c_0$.
We prove that whenever $\delta \leq B$, then there is a greedy poly-time algorithm guaranteeing an approximation factor $\rho$ depending on $\delta$. 
A surprisingly sharp \emph{transition phase} occurs, instead, when $B< \delta$. In this case, no poly-time approximation algorithm is possible, unless $\mathsf{P} = \mathsf{NP}$, even when the approximation factor is a function in the size of the problem. 
An interpretation of such a result is that an optimal manipulation can be found in polynomial time only if every single voter is sufficiently uncertain so that her final decision on $c_0$ depends on the received messages. 
Instead, if there are voters for which $c_0$ cannot be made to be the most preferred candidate due to the budget constraint, then manipulation is unaffordable in polynomial time. 
This result poses a severe obstacle to the possibility for a manipulator to successfully alter the outcome of an election, as it is likely that there are voters who never change opinion in real-world scenarios (\emph{e.g.}, the candidates' supporters). 
Even more importantly, we show that this hardness result does not hold merely for worst-case (thus, potentially, knife-edge or rare) instances. Indeed, a large class of algorithms (including most of the approaches recently adopted for social-influence problems) fail to compute an empirically bounded approximation even on elementary networks, as undirected graphs with every node having a degree at most two or directed trees. 
Furthermore, the hardness holds even on simple graphs, proving that maximizing the increase of margin of victory is \NPHard \ even on graphs as lines, and we discuss how our results extend to variants of our model.

In the case of edge removal/addition, the characterization is more intricate. We study both the case with only two candidates and the one with multiple candidates, and, in the latter case, we study both the  subcase with single-news-article messages and the more realistic subcase with arbitrary messages. 
We show that, in any of these cases, the problem of deciding whether a set of edges to remove/add in the network exists to make the desired candidate win is hard even when the score distances are unitary.
Surprisingly, these results hold even if the manipulator has an unlimited budget of edges to remove or add, except for the trivial setting in which there are two candidates and messages are single-news-article.
% that admits a constant-factor approximation even with arbitrary score distances.
In this latter case, the optimal solution when the budget is unlimited is to remove all edges, if the messages are against the desired candidate, or to add all possible edges, otherwise. For the remaining cases, we formally prove that it is hard to find a set of edges to remove or add  that causes an increment in the margin of victory of the desired candidate that is a constant (and, in some case, even exponential) approximation of the best possible increment that can be achieved. Our results still hold with acyclic networks.

Incidentally, in order to establish these results, we also provide new results for the basic Influence Optimization problem, that consists in maximizing or minimizing the number of nodes that receive the information spread over the network.~\footnote{In \cite{Sheldon2010}, the problem of adding edges to arbitrary nodes of the networks is studied, proving that the function is not submodular. In \cite{Khalil2014}, two types of graph modification are investigated, adding/removing edges in order to minimize the information diffusion showing that this network structure modification problem has a supermodular objective. Heuristics for the edge removal problem have been studied in \cite{kimura2008,kuhlman2013}.  However, no hardness results are known.} We, indeed, originally prove that the minimization (maximization, respectively) variant of the problem cannot be approximated within any constant factor by removing (adding, respectively) a limited number of edges.~\footnote{For the sake of completeness, we mention that the Influence Optimization problem has been widely investigated when the manipulator makes seeding~\cite{kempe2003maximizing}.}

The hardness results presented in this work are a starting point for shaping the landscape of manipulability of election through social networks. This task is fundamental to understand when and how one must design interventions to reduce the severe effects of the spread of misinformation. Although our results are positive, showing that manipulation is not affordable in the worst case, we believe that the border of manipulability can be further sharpened. We here present a seminal study along this direction, looking at manipulators that face a \emph{repotimization problem}~\cite[Chap. 4]{ausiello2012complexity} and thus answering the question ``is manipulation easier if a solution to the problem for a given instance is already available and a local modification occurs?''. Note that this is very common in the real world, in which the social relationships among voters remain essentially stable between an election and the next one. Surprisingly, we show that all our hardness results are \emph{robust} to the knowledge of solutions in similar settings since they still hold in this reoptimization setting.

\subsection{Structure of the Paper}
The paper is structured as follows. Section~\ref{sec:model} formally introduces the model and the computational problems we study. Section~\ref{sec:seeding} provides our main results on the seeding problem.
% , and Section~\ref{sec:extensions} discusses how our results can be used for some extensions of our model.
Section~\ref{sec:edgeremoval} provides results on edge removal, while Section~\ref{sec:edgeaddition} provides results on edge addition. Section~\ref{sec:reoptimization} discusses the robustness of our hardness results in the case of reoptimization. Finally, Section~\ref{sec:conclusions} concludes the papers and describes future research directions. For the sake of presentation, some proofs are provided in~\ref{sec:appendix}.

\section{Model and Problem Statement}
\label{sec:model}
We have a set of candidates $C=\{c_0,c_1,\dots,c_\ell\}$ and a network of voters, represented as a weighted directed graph $G = (V,E, p)$, where $V$ is the set of voters, $E$ is the set of direct edges, and $p \colon V \times V  \rightarrow [0,1]$ denotes the strength of the potential influence among voters. 
In particular, for each edge $(u,v)$ where $u,v \in V$, $p(u,v)$ returns the strength of the influence of $u$ on $v$.% (more details on the influence model are discussed below).

Each voter $v$ assigns a \emph{score}, by function  $\pi_v:C\rightarrow \mathbb{N}$,  to every candidate $c_i$. 
We assume function $\pi_v$ to be injective, thus returning a different score to every candidate, formally, $\pi_v(i) \neq \pi_v(j),  \forall c_i,c_j \in C$.
The score $\pi_v(i)$ models how much voter~$v$ likes candidate $c_i$ and induces, for voter $v$, a strict preference ordering over the candidates.
Thus, $\pi_v(i) > \pi_v(j)$ models that voter $v$ (strictly) prefers $c_i$ to $c_j$. 
Sometimes we will denote with $\langle \pi_v(0),\dots,\pi_v(\ell) \rangle$ the score vector of voter~$v$.
We will say that the scores have \emph{unitary score distances} if $\pi_{v}(i) \in \{0, \ldots, \ell-1\}$ for every voter $v$ and for every candidate $c_i$.

The election is based on \emph{plurality voting}, where every voter casts a single vote for a single candidate, and the candidate that received the largest number of votes wins the election. 
We assume voters to be \emph{myopic}, casting a vote for the candidate with highest score in their preference ordering.
For each candidate $c \in C$, we denote with $V_c$ the set of voters that rank $c$ as first, formally, $V_{c} = \left\{v \in V \mid c = \argmax_{c_i \in C} \pi_v(i) \right\}$.

Let $S \subseteq V$ be a subset of voters said \emph{seeds}. Every seed $s$ can be selected to initiate the diffusion of information about multiple candidates.
%
%We denote with $m_s=(q_0,...,q_{\ell})$ the \emph{message profile} of $s \in S$, where $q_i \in \{+,\cdot,-\}$, with $q_i = +$ ($q_i = -$, respectively) representing that $s$ sends a \emph{positive} (\emph{negative}, respectively) message on  $c_i$, and  $q_i = \cdot$ representing that $s$ does not send any message on $c_i$.
% \footnote{
%A positive message on a candidate increases the scores that voters assign to that candidate, while a negative message does the reverse.
% }
%
%We use $m_s(i) \in \{+,\cdot,-\}$ to denote the message sent by $s$ on candidate $c_i$ and $M= \cup_{s \in S}m_s$ to denote the set of all the messages. We denote with $|m_s|$ the size of messages sent by $s$, i.e., the number of positive or negative message sent by these seed. Similarly, $|M| = \sum_{s \in S} |m_s|$.
%
%
We denote with $m_s=(q_0,...,q_{\ell})$ the \emph{message} of $s \in S$, where $q_i \in \mathbb{Z}$, with $q_i > 0$ ($q_i < 0$, respectively) representing that $s$ initiates the diffusion of $q_i$ \emph{positive} (\emph{negative}, respectively) news articles on $c_i$, and  $q_i = 0$ representing that $s$ does not send any information about $c_i$.
% \footnote{
Positive information on a candidate increases the scores that voters assign to that candidate by $|q_i|$, while negative information does the reverse.
% }
%
We use $m_s(i) \in \mathbb{Z}$ to denote information sent by $s$ on candidate $c_i$ and $M= \cup_{s \in S}~m_s$ to denote the whole information sent by seeds. We denote with $|m_s|=\sum_{c_i \in C} |m_s(i)|$ the number of news articles sent by $s$. Similarly, $|M| = \sum_{s \in S} |m_s|$.
If $m_s = m_{s'}$ for every pair of seeds~$s,s' \in S$, and, for all the candidates $c_i$ except $c_j$, it holds $m_s(i)=0$, while for $c_j$ it holds $m_s(j)\in \{+1,-1\}$, then we say that we are in the setting with \emph{single-news-article messages}.
% \textcolor{orange}{Formally, the setting with  messages are such that $m_s = m_{s'}$ for every pair of seeds~$s,s'$  and, .}

\subsection{Diffusion Model}
Given a pair of seeds/messages $(S, M)$, messages are supposed to spread over the network according to a \emph{multi-issue independent cascade} (MI-IC) model.
% \cite{bharathi2007competitive}\footnote{Actually, the model in \cite{bharathi2007competitive} focuses on a slightly different setting, namely viral marketing, in which alternative products, in place of messages, diffuse competitively over the network. Hence, our model differs from the one in \cite{bharathi2007competitive} since a node  can send multiple messages. Apart from that, we keep all the main feature of the model by \cite{bharathi2007competitive}: it is based on independent cascade, and once a node is activated it ceases to receive messages in next rounds.}. 
%
Roughly speaking, in this model, each seed $s$ propagates the message $m_s$ to her neighbors. 
Then, a voter~$v$ with $v \not \in S$, receiving a message from $s$, accepts the information that this message carries with probability $p(s,v)$. 
If voter $v$ accepts the message, we say that $v$ is \emph{activated} by~$s$. 
In her turn, each just activated voter $v$ sends the received messages to her neighbors $u$ that can activate with probability $p(v,u)$ if not activated in the past  and, then, voter $v$ becomes \emph{inactive}. 
The process continues as long as there is some active voter, and it is repeated for every message $m_s$ sent by one of the seeds. 

Formally, given graph $G = (V, E, p)$, we define the \emph{live-graph} $H = (V, E')$, where each edge $(u,v) \in E$ is included in $H$ with probability $p(u,v)$.
Moreover, for every $s \in S$, we introduce a set $A^t_{m_s}\subseteq V$ composed of the \emph{active} voters at time $t$ due to message $m_s$.
Every set  $A^t_{m_s}$ is initialized with the seed sending the corresponding message for $t=0$, \emph{i.e.}, $A^0_{m_s} = \{s\}$, and the empty set for $t>0$.
At every time $t \geq 1$, set $A^t_{m_s}$ is defined as follows: for every edge $(u,v) \in E'$, we consider the set $\mathcal{M}_{(u,v)} \subseteq M$ of messages $m_s$ such that $u \in A^{t-1}_{m_s}$---and thus $u$ has just been  activated by $m_s$---and $v \notin \bigcup_{i < t} A^{i}_{m_s}$---and thus $v$ has never been activated by $m_s$; then for each $(u,v)$ such that $\mathcal{M}(u,v)$ is not empty, we add $v$ to $A^t_{m_s}$ for every $m_s \in \mathcal{M}(u,v)$.
The diffusion process of message $m_s$ terminates at time  $T_{m_s}$ when $A^{T_{m_s}}_{m_s} = \emptyset$.
Finally, the cascade terminates when the diffusion of every message $m_s$ terminates. A voter that activates at some $t$ is said \emph{influenced}.
Note that, when the messages are single-news-article, there are two candidates, and the score distances are unitary, this process reduces to the renowned independent cascade model~\cite{kempe2003maximizing}.

\subsection{Preference Revision} 
When a voter $v$ accepts a message received by a neighbor, her preferences can change.
%Specifically, we assume that for every voter $v$ and every candidate $c$ there is a \emph{window size} $w^+_u(c)$ ($w^-_u(c)$, respectively), that denote the maximum number of positive (negative, resp.) messages about $c$ that $u$ will use for preference revision. This model a sort of diminishing return over the received messages.
%
Let us now denote with $R \subseteq M$ a set of received messages. A \emph{ranking revision function} $\phi$ associates each pair $(\pi,R)$ with a new ranking $\pi'$ obtained by revising ranking $\pi$ according to the set of received messages $R$. 
We use a \emph{score-based} ranking revision function in which a positive (negative) message $m_s$ on a candidate $c_i$ increases (decreases) her score by $m_s(i)$.
% $1+\epsilon$, where $\epsilon$ is an arbitrary small constant, e.g., $\epsilon <\frac{1}{|V|}$.
Formally, each voter updates every candidate's score as follows:~\footnote{It is easy to see that every hardness result related to this model keeps to hold even when we allow the same message causes a different score increment (decrement) to different voters, or if received by different neighbors, or if sent by different seeds.} 
\[
\pi_v(i)\leftarrow \pi_v(i)+  \sum_{m_s \in R}  m_s(i).
\]
According to our assumption on $\pi_v$, we require that, at the end of the diffusion process, no pair of candidates $c_i, c_j$ has the same value of $\pi_v$. We can obtain such a property, \emph{e.g.}, by breaking ties according to some rule and slightly tilting scores so that they satisfy the tie-break outcome.
For the sake of simplicity, we brake ties in favor of the candidate ranked last before the diffusion process.
Such a tie-breaking rule can be obtained by slightly perturbing the initial score with a multiplicative factor $(1-\epsilon)$, where $\epsilon$ is a sufficiently small positive constant, \emph{e.g.}, $\epsilon=\frac{1}{1 + \max_{v,i} \pi_v(i)}$, and then apply the update rule as
\[
	\pi_v(i)\leftarrow (1-\epsilon) \,\pi_v(i) +  \sum_{m_s \in R}  m_s(i).
\]

Given a seed set $S$, a set $M$ of messages, a set $E$ of edges, and a live graph $H$, $\pi^*_v(i,S,M,E,H)$ denotes for every voter $v \in V$, the score of candidate $c_i$ at the end of the MI-IC diffusion (\emph{i.e.}, after the preference revision).
Moreover, for each candidate $c \in C$, we denote with $V^*_c$ the set of voters for which $c$ is ranked as first after the preference revision, \emph{i.e.}, $V^*_{c}(S,M,E,H) = \left\{v \mid \argmax_{c_i} \pi_v^*(i,S,M,E,H) = c\right\}$. 
Finally, we define the \emph{margin of victory} $\MoV$ of $(S, M, E,H)$ as
\[
\MoV(S,M,E,H) = \Big|V^*_{c_0}(S,M,E,H)\Big| - \\ \max_{c \neq c_0} \Big|V^*_{c}(S,M,E,H)\Big|.
\]
Given a live-graph $H$, \MoV\ returns the number of votes that $c_0$ needs to win the election, if the first term is smaller than the second, and the advantage of $c_0$ with respect to the second-best ranked candidate, otherwise.

Next we provide an example of the concept defined above.
\begin{example} Consider Figure \ref{fig:example}, depicting the connections among five voters.
\begin{figure}[tb]
	\centering
	\includegraphics[width=0.85\linewidth]{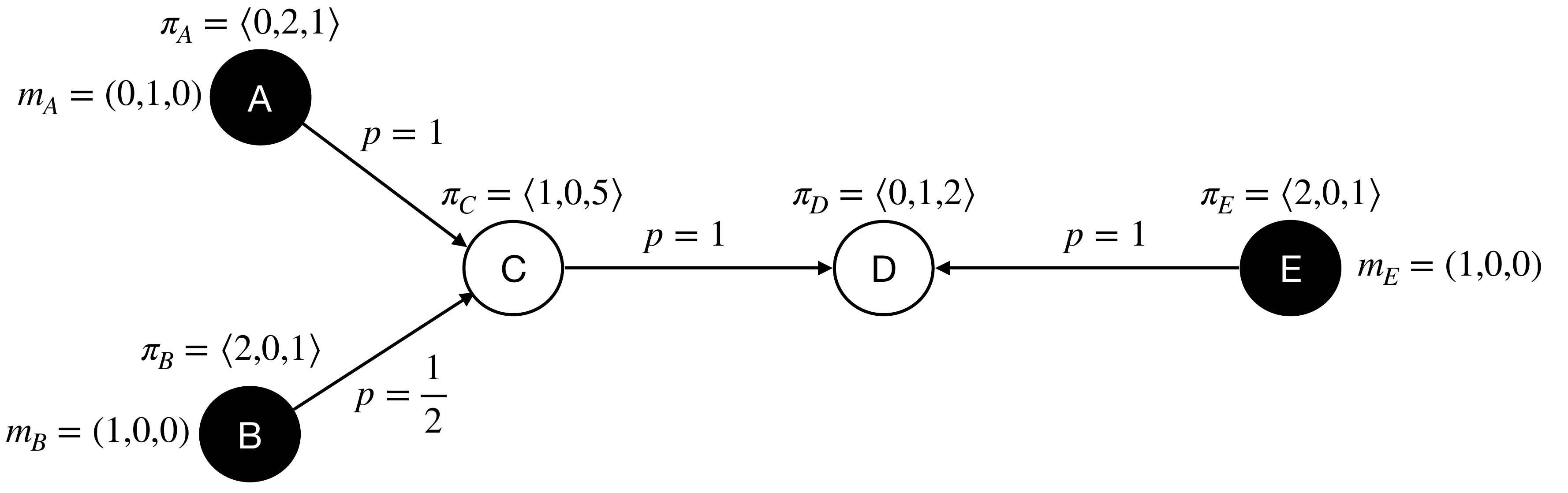}
	\caption{Example of an election with three candidates $c_0,c_1,c_2$. Black nodes represent seeds: node A sends a positive message on $c_1$, while nodes B and E send a positive message on $c_0$. The tuples $\langle \pi_v(0),\pi_v(1),\pi_v(2)\rangle$ above the nodes are the voters' preferences.}
	\label{fig:example}
\end{figure}
%
%We assume that $\sigma(c) = 1 + \varepsilon$ for every candidate $c$, where $\varepsilon$ is an arbitrary small constant, e.g., $\varepsilon <\frac{1}{|V|}$\footnote{This is equivalent to say that $\sigma(c) = 1$ for every $c$ and ties are broken in favor of the candidate with a lower initial ranking.}.

Two different live-graphs $H_1$ and $H_2$ are possible depending on whether or not B influences C. 
This happens with probability $\frac{1}{2}$.
%
%

%\textcolor{purple}{***qui sotto ci vogliono le parentesi angolari?***}
%In $H_1$, B does not influence C and C receives only a positive message on $c_1$. She increases the score of $c_1$ by $1+\epsilon$. However, voter C as a very high evaluation of candidate $c_2$ and will continue to prefer $c_2$ over $c_0$ and $c_1$.
%Node D updates her score to $(1,2+\epsilon,2)$ and will vote for $c_1$.
%Thus, if B do not diffuse the message, at election time $c_0$ has $2$ votes (B and E), $c_1$ has two votes $(A,D)$ and $c_2$ has one votes ($C$). Hence, $\MoV(S,M,E,H_1)=2-\max\{2,1\}=0$.
In $H_1$, B does not influence C and C receives only a positive news article on $c_1$, thus increasing the score of $c_1$ by $1$. 
However, C has a very high evaluation of candidate $c_2$ and keeps to prefer $c_2$ over $c_0$ and $c_1$.
Instead, D updates her score to $\langle1,2-\epsilon,2-2\epsilon\rangle$ and votes for $c_1$.
Thus, at election time, $c_0$ has $2$ votes (B and E), $c_1$ has $2$ votes (A and D) and $c_2$ has one votes (C), and therefore $\MoV(S,M,E,H_1)=2-\max\{2,1\}=0$.

%In $H_2$, B influences C and C receives a positive message on $c_0$ and a positive message on $c_1$. However, she continues to prefer $c_2$ over $c_0$ and $c_1$.
%Voter D receives a positive message on $c_1$ and two positive messages on  $c_0$, updates her score to $(2+2\epsilon,2+\epsilon,2)$ and votes for $c_0$.
%Hence, $\MoV(S,M,E,H_1)=3-\max\{1,1\}=2$.

In $H_2$, B influences C and C receives a positive news article on $c_0$ and a positive news article on $c_1$. However, C keeps to prefer $c_2$ over $c_0$ and $c_1$.
Voter D receives a positive news article on $c_1$ and two positive news articles on  $c_0$, thus updating the scores to $\langle2,2-\epsilon,2-2 \epsilon\rangle$ and then voting for $c_0$.
Hence, $\MoV(S,M,E,H_2)=3-\max\{1,1\}=2$.
\end{example}

\subsection{Election Control Problem}
The election control problem involves a single  agent (\emph{i.e.}, the manipulator) whose objective is to spend a \emph{budget} $B$ to make $c_0$ win the election. We consider two different manipulation strategies: seeding and network modification by edge removal/addition. For the sake of simplicity, we assume that the cost incurred by the manipulator for seeding is one for every single news article sent by each seed and therefore the cumulative cost for seeding is $|M|$, while the cost for network modification is equal to the number of removed/added edges. We study each form of manipulation singularly. Nevertheless, it is easy to see that the results also extend to the case with multiple simultaneous forms of manipulation.
%can \emph{both} remove and add edges. \footnote{It reduces to the case where the manipulator can only remove edges setting $p(\cdot)=0$ for all the added edges. }

We next formally state the problems we study in the paper.
\begin{definition}[\textsc{Election-Control-by-Seeding (ECS)}]
	Given an election scenario $(C,G,\{\pi_v\})$ and budget $B\in \mathbb{N}$, the goal is finding a set $S$ of seeds and messages $m_s$ for every $s \in S$, with $|M| \le B$,  to maximize $\Expec_H[\Delta_\MoV^S(S,M,H)]$, where $\Delta_\MoV^S(S,M,H) = \MoV(S,M,E,H) - \MoV(\emptyset,\emptyset,E,H)$ is the  increase of \MoV\ due to the messages send by seeds S.
\end{definition}

\begin{definition}[\textsc{Election-Control-by-Edge-Removal (ECER)}]
	Given an election scenario $(C,G,\{\pi_v\},S,M)$ and budget $B\in \mathbb{N}\cup \{\infty\}$, the goal is finding $E' \subseteq E$ with  $|E'|\leq B$ to remove from graph $G$ to maximize $\Expec_H[\Delta_\MoV^-(E',H)]$, where $\Delta_\MoV^-(E',H) = \MoV(S,M,E \setminus E',H) - \MoV(S,M,E,H)$ is the  increase of \MoV\ due to the removal of edges $E'$.
\end{definition}
\begin{definition}[\textsc{Election-Control-by-Edge-Addition (ECEA)}]
	Given an election scenario $(C,G,\{\pi_v\},S,M)$ and budget $B\in \mathbb{N}\cup \{\infty\}$, the goal is finding $E'$ with  $E'\cap E = \emptyset$ and $|E'|\leq B$ to add to $G$ to maximize $\Expec_H[\Delta_\MoV^+(E',H)]$, where $\Delta_\MoV^+(E',H) = \MoV(S,M,E \cup E',H) - \MoV(S,M,E,H)$ is the increase of  \MoV\ due to the addition of edges $E'$.
\end{definition}

An algorithm $A$ is said to always return a $\rho$-approximation for an ECS problem with $\rho \in [0,1]$ potentially depending on the size of the problem, if, for each instance of the problem,
% (defined by the set of candidates $C$, the network of voters $G$, their rankings $\{\pi_v\}_v$, and the budget $B$)
it returns a feasible $(S, M)$ such that $\Expec_H[\Delta_\MoV(S,M,H)] \geq \rho\, \Expec_H[\Delta_\MoV(S^*,M^*,H)]$. A similar definition holds for all the other optimization problems.

% In the following, when we refer to the \emph{single-message case}, we are considering the case in which all the messages are the same candidate and are either all positive or all negative.
% %
% %Formally, there is a $c_i \in C$ and a $q_i$ such that $m_s(j) \neq 0$ implies $j=i$ and $m_s(i) = t \ q_i, t \in \mathbb{N}$.
% %a single pair $(s,i)$ such that $m_s(i)\neq \cdot$,
% We use the term \emph{multi-message case}, otherwise.

\subsection{Influence Optimization}
Incidentally, our analysis of the ECER and ECEA problems allow us to provide results also on the (more general) influence maximization/minimization problems when the manipulator can either \emph{remove} or \emph{add edges}. 
To formally describe these problems, we need to define function $\chi:S\times E \times H \rightarrow \mathbb{R}_+$ returning the number of influenced nodes with seeds $S$, edges $E$ and live graph $H$. When the set of edges $E$ is fixed (\emph{e.g.}, in seeding), we will use $\chi(S,H)=\chi(S,E,H)$, removing the dependence from $E$. Finally, with abuse of notation, we also define $\chi(S,E)= \Expec_H[\chi(S,E,H)]$. 
We have the following two problems.
\begin{definition}[\textsc{Influence-Minimization-by-Edge-Removal (IMER)}]\phantom{a}
	Given a setting $(G,S,M)$ and budget $B\in \mathbb{N}\cup \{\infty\}$, the goal is finding a set $E' \subseteq E$ with  $|E'|\leq B$ to remove from graph $G$ to maximize $\Delta I^-(E') = \chi(S,E)-\chi(S,E \setminus E')$.
\end{definition}

\begin{definition}[\textsc{Influence-Maximizationn-by-Edge-Addition (IMEA)}]\phantom{a}
	Given a setting $(G,S,M)$ and budget $B\in \mathbb{N}\cup \{\infty\}$, the goal is finding a set $E'$ with  $E'\cap E = \emptyset$ and $|E'|\leq B$ to add to graph $G$ to maximize $\Delta I^+(E') = \chi(S,E \cup E') - \chi(S,E)$.
\end{definition}

\section{Seeding Complexity}
\label{sec:general}
\label{sec:seeding}
We characterize the computational complexity of the ECS problem. Unless specified otherwise, the results provided in this section refer to the general setting when both messages and score distances are arbitrary. Our characterization is based on the parameter $\delta=\max_{v \in V, c_i \neq c_0} \left\{\pi_v(i)-\pi_v(0)\right\}$, representing the cost the manipulator needs to spend to convince the most reluctant voter to vote for $c_0$.
We introduce the following definition that is useful to describe the hardest instances of the ECS problem.

\begin{definition}
	An ECS problem instance is said \emph{hard to manipulate} if $B < \delta$.
\end{definition}

\subsection{Inapproximability Results}

We introduce the \textsc{Set-Cover} problem, that is well known to be $\mathsf{NP}$-hard, to prove the hardness of ECS.

\begin{definition}[\textsc{Set-Cover}]
	Given a set $N = \{z_1,\dots,z_n\}$ of $n$ elements, a collection $X= {x_1,\dots,x_g}$ of sets with $x_i \subset N$, and a positive integer $h$, the objective is to select a collection $X^* \subset X$, $|X^*| \le h$ with $\cup_{x_i \in X^*} \,x_i=N$.
\end{definition}

\begin{theorem}
\label{thm:inapprox}
Given the set of ECS instances said hard to manipulate and with at least three candidates, for any $\rho > 0$ even depending on the size of the problem, there is not any poly-time algorithm returning a $\rho$-approximation to the ECS problem, unless $\mathsf{P} = \mathsf{NP}$.
\end{theorem}

\begin{proof}
	The proof uses a reduction from \textsc{Set-Cover}.
	Given an instance of \textsc{Set-Cover}, we build an instance of the election control problem with $3$ candidates as follows.~\footnote{If $|C|>3$, we set $\pi_v(i)=\max\{\pi_v(0),\pi_v(1),\pi_v(2)\}-B-1, \forall i \in \{3,\dots,C-1\}, v \in V$.}
	The voters' network $G$, showed in Figure~\ref{fig:reduction}, consists of three disconnected components, that we denote as $G_1$, $G_2$, and $G_3$.
	Note that all edges of $G$ have $p(u,v) = 1$.
	
	\begin{figure}
		\begin{center}
		\includegraphics[width=0.9\linewidth]{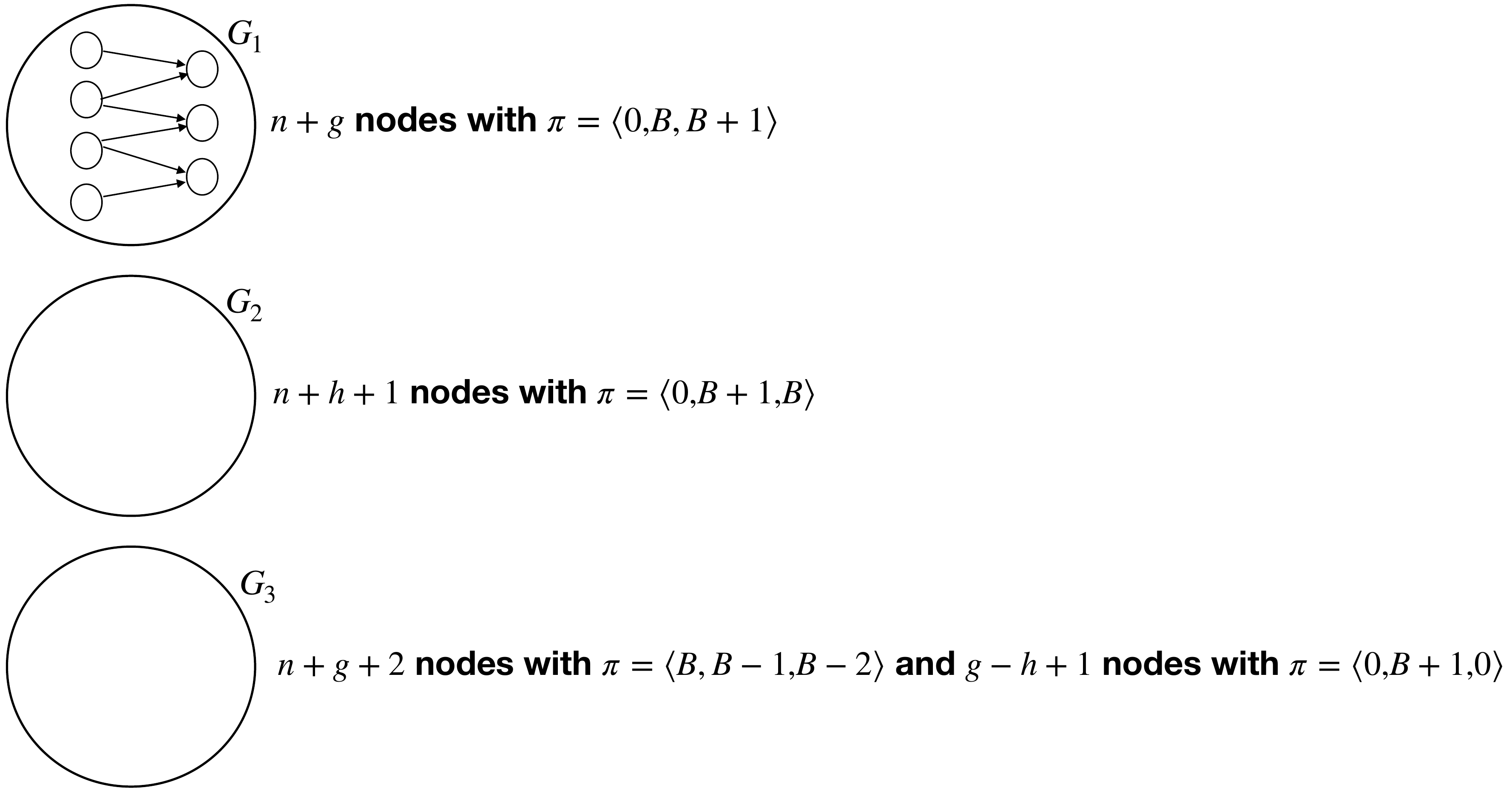}
		\end{center}
		\caption{Structure of the election control problem used in the proof of Theorem~\ref{thm:inapprox}.}
		\label{fig:reduction}
	\end{figure}
	
	We set the budget $B=h+1$. 
	%It is easy to see that we can focus w.l.o.g. to instances with $B>\ell$. 
	The component $G_1$ has $g+n$ nodes and it is used to model the \textsc{Set-Cover} instance. Indeed, for each $z_i \in N$, we have in $G_1$ a node $v_{z_i}$;
	moreover, for each $x_i \in X$, we have in $G_1$ a node $v_{x_i}$
	with an edge toward $v_z$ for each $z \in x_i$.
	The preferences of all voters $v$ corresponding to nodes in $G_1$ are: $\pi_{v}(0)=0$, $\pi_{v}(1)=B$, $\pi_{v}(2)=B+1$.
	
	%The component $G_2$ is a clique of $(3\ell-2)(m+n)+m+\ell-h-1$ nodes, such that
	%\begin{itemize}
	%	\item $2(m+n)$ nodes have ranking $\pi_{v}=\langle0,B,B+1,B-3,\dots,B-\ell\rangle$;
	%	\item $2(m+n)+m-h+1$ nodes have ranking $\pi_{v}=\langle0,B+1,B,B-3,\dots,B-\ell\rangle$;
	%	\item for every $i \not \in \{0,1,2\}$, $3(m+n)+1$ nodes have ranking $\pi_v(i)=B+1$, $\pi_v(0)=0$, and $\pi_v(j)=B-j$ for each $c_j \notin  \{c_0,c_i\}$.
	%\end{itemize}
	
	The component $G_2$ is a clique of $n+h+1$ nodes with preferences: $\pi_{v}=\langle0,B+1,B \rangle$.
	The component $G_3$ is a clique of $n+2g-h+3$ nodes, such that $n+g+2$ nodes have a preference ranking $\pi_{v}=\langle B,B-1,B-2 \rangle$ and $g-h+1$ nodes have a preference ranking $\pi_{v}=\langle 0,B+1,0 \rangle$.
	
	Note that $|V_{c_0}| = g+n+2$, $|V_{c_1}| = g+n+2$, and $|V_{c_2}| = g+n$. Hence, $\MoV(\emptyset,(),E,H) = 0$.
	
	We next prove that this instance allows a feasible solution $(S^*, M^*)$ with $\MoV(S^*,M^*,E,H) > 0$ if and only if there is a solution of the \textsc{Set-Cover} instance of size at most $h$.

	\textbf{(If)} Let $X^* \subseteq X$ be the solution of \textsc{Set-Cover} of size $h$ (\emph{i.e.}, $|X^*| = h$ and $\cup_{x_i \in X^*}=N$).~\footnote{If there is a solution of \textsc{Set-Cover} $X^*$ of size less $h$, then we can achieve a solution of \textsc{Set-Cover} of size exactly $h$, by padding $X^*$ with arbitrary element in $X \setminus X^*$.}
	Then we set $(S^*, M^*)$ as follows: for every $x_i \in X^*$, we include $v_{x_i} \in S^*$ and we set $m^*_{v_{x_i}}$ such that $q_0 = 0$, $q_1=1$, and $q_2 = 0$; moreover, we include in $S^*$ an arbitrary node $v \in G_2$ and we set $m^*_v$ such that $q_0=0$, $q_1=0$, and $q_2 = 1$.
	
	From the above arguments, it directly follows that $(S^*, M^*)$ is feasible.
	We next show that $\MoV(S^*,M^*,E,H) > 0$.
	Indeed, the diffusion of messages leads each voter corresponding to nodes in $G_2$ to prefer $c_2$ to $c_1$. Moreover, the dynamics leads $h + n$ voters in $G_1$ (\emph{i.e.}, the seeds and the ones corresponding to elements $z_i \in N$) to prefer $c_1$ to $c_2$.
	Hence, $|V^*_{c_1}(S^*,M^*,H)| = |V_{c_1}| - |G_2| + h+n = g+n+2- n-h-1  +n+h = g+n+1$,
	and $|V^*_{c_2}(S^*,M^*,H)| = |V_{c_2}| + |G_3| - h-n = g+n+1$.
	Hence, $\MoV(S^*,M^*,E,H) = 1$, as desired.
	
	\textbf{ (Only if)} Suppose that there exists a pair $(S^*, M^*)$ such that it holds $\MoV(S^*,M^*,E,H) > 0$. Note that, since $\delta>B$, $c_0$ cannot gain votes and  $|V^*_{c_0}(S^*,M^*,H)| = |V_{c_0}|$. Hence, in order to have $\MoV(S^*,M^*,E,H) > 0$, it must be the case that the number of voters whose most-preferred candidate is $c_1$ decreases by at least one unit and the number of voters whose most-preferred candidate is $c_2$ increases by at most one unit.
	
	Since, $c_2$ has to take at least one vote and she cannot take votes in $G_3$, $c_2$ must take voters in $G_2$. Since $G_2$ is a clique, it must be that all votes of $c_1$ are taken by candidate $c_2$.
	
	Thus, $c_1$ loses all its voters in $G_2$ in favor of $c_2$. Note that a single message is sufficient (a positive message for $c_2$) to this aim. However, this implies that $c_2$ must lose $n+h$ voters in $G_1$, otherwise $|V^*_{c_2}(S^*,M^*,E,H)| > g+n+n+h+1-(n+h)$ and thus $\MoV(S^*,M^*,E,H) \leq 0$, that contradicts our hypothesis. Observe that these votes must be necessarily lost in favor of $c_1$.
	
	Hence, we are left with $h$ available messages to make $n+h$ voters to change their vote from $c_2$ to $c_1$. Observe that, in order to make a voter to change, it is sufficient a single message (a positive message for $c_1$). However, if less than $h$ seeds sending this message are located among nodes $v_{x_i}$ for $x_i \in X$, then less than $n+h$ voters will change their mind (since nodes $v_{x_i}$ for $x_i \in X$ have no in-going edges).
	
	Finally, we must have that the $h$ seeds in $G_1$ are neighbors of every node $v_{z_i}$ for $z_i \in N$. Hence, the set $X^* = \{x_i \colon v_{x_i} \in S^*\}$ has size $h$ and, by construction of $G_1$, $\bigcup_{x \in X^*} x = N$, \emph{i.e.}, $X^*$ is a solution of \textsc{Set-Cover} of size at most $h$.
	
	Hence, we can conclude that a feasible solution $(S^*,M^*)$ satisfying the property $\Delta_{\MoV}(S^*,M^*,H) > 0$ exists if and only if a solution for the \textsc{Set-Cover} instance exists. Note also that if a solution with $\Delta_{\MoV}(S^*,M^*,H) > 0$ exists, then $\Delta_{\MoV}(S,M,H) > 0$ even for any $\rho$-approximate solution $(S,M)$, regardless of the value of $\rho$. Thus, if a polynomial time $\rho$-approximate algorithm for election control problem exists, then the \textsc{Set-Cover} problem can also be solved in poly-time, implying that $\mathsf{P} = \mathsf{NP}$.
\end{proof}

Theorem~\ref{thm:inapprox} essentially states that there is no chance that a manipulator designs an algorithm allowing her to maximize the increment in the margin of victory of the desired candidate in the set of instances in which there is no way for making $c_0$ become the most preferred of any voter. However, Theorem~\ref{thm:inapprox} does not rule out that the worst-case instances are very rare and/or knife-edge. However, we show that simply algorithms will fail even on very simple instances. Specifically, we show that if the manipulator greedily chooses the messages to send, then her approach fails even for simple graphs, namely graphs with all nodes having a degree two or trees.

In details, given a set $S$ of seeds and corresponding messages $M$,
we denote as $\mathcal{F}(S,M)$ the set of pairs $(s, m_s)$, with $s \notin S$ such that either
 $$\Expec_H \left[\MoV(S \cup \{s\},(M, m_s), E,H)\right] > \Expec_H \left[\MoV(S,M, E,H)\right]$$
 or 
 $$\Expec_H \left[V^*_{c_0}(S \cup \{s\},(M, m_s),E, H)\right] > \Expec_H \left[V^*_{c_0}(S,M,E, H)\right].$$
 That is, $\mathcal{F}(S,M)$ includes all the ways of augmenting a current solution so that either the margin of victory of $c_0$  or the number of her votes increases.
 Then, we say that an algorithm for the election control problem uses the \emph{greedy} approach, if it works as follows:
% \begin{itemize}
%  \item 
 it starts with $S=\emptyset$ and $M=()$;
%  \item 
 until the set $\mathcal{F}(S,M)$ is not empty, choose one $(s, m_s) \in \mathcal{F}(S,M)$ and set $S = S \cup \{s\}$ and $M=(M,m_s)$.
% \end{itemize}
%
We show that every algorithm in this class fails even for elementary networks (the proof of the following proposition is provided in~\ref{sec:appendix}).

\begin{restatable}{proposition}{propositionone}
 \label{prop:example1}
  For any $\rho > 0$ even depending on the size of the problem, no algorithm following the greedy approach  returns a $\rho$-approximation to the ECS problem, even in undirected graphs in which each node has degree at most~$2$.
\end{restatable}

A similar result holds even by considering directed trees (the proof of the following proposition is provided in~\ref{sec:appendix}).

\begin{restatable}{proposition}{propositiontwo}
 \label{prop:example2}
 For any $\rho > \frac{38}{|V|}$,
% % even depending logarithmically in the size of the problem, 
no algorithm following the greedy approach returns a $\rho$-approximation to the ECS problem, even in directed trees.
 \end{restatable}

We recall that the greedy algorithms are essentially the only known algorithms guaranteeing bounded approximations for many problems related to the election control problem, such as the well-known influence maximization problem~\cite{kempe2003maximizing}. Hence, even if an algorithm exists enabling the manipulator to control the election in many instances,
% Proposition~\ref{prop:example1} and Proposition~\ref{prop:example2}
the propositions above show that new approaches are necessary to design it.

Now, we provide further evidence of the hardness of the problem even in simple graphs, showing that  maximizing the expected $\Delta_\MoV^S$ is \NPHard \ even on a line.
We recall that, while Influence Maximization by Seeding is \NPHard \ with arbitrary graphs, there exists a polynomial-time algorithm when the graph is a line~\cite{wang2016bharathi}. The proof of the following theorem is provided in~\ref{sec:appendix}.

\begin{restatable}{theorem}{theoremtwo}
\label{thm:line}
	The ECS problem with at least four candidates is \NPHard  \ even on line graphs.
\end{restatable}

We conclude the section proving that when we restrict to single-news-article-messages the ECS problem is hard even in the simple model in which there are two candidates and the score distances are two. We reduce from \textsc{Densest-k-Subgraph} whose definition follows.

\begin{definition}[\textsc{Densest-k-Subgraph (DkS)}]{
		Given an indirect graph $G=(X,N)$, find the set $X^*$ of $k$ vertexes that maximizes $d(X^*)=|E'|$, where $G(X^*)=(X^*,E')$ is the subgraph of $G$ with vertexes $X^*$.
	}
\end{definition}
The proof of the following theorem is provided in~\ref{sec:appendix}.
\begin{restatable}{theorem}{theoremthree}
	\label{thm:DkS}
	If there is a $\rho>0$ approximation algorithm for the ECS problem with single-news-article messages, two candidates, and arbitrary scores, then there is a $\rho$-approximation algorithm for \textsc{DkS}.
\end{restatable}

Manurangsi shows that there is no constant-factor polytime approximation algorithm for the \textsc{DkS} problem unless the Exponential Time Hypothesis is false~\cite{DBLP:conf/stoc/Manurangsi17}. Therefore, it is unlikely that the specific ECS problem considered above is approximable within a constant factor. This is in stark contrast with the known constant approximation algorithm existing in the setting when we further constrain score distances to be unitary \cite{wilder2018controlling}.

\subsection{Approximation Results}
We next show that the condition we used in the previous section to characterize hard instances is tight. Indeed, by dropping that condition, we can design poly-time approximation algorithms. Moreover, these algorithms turn out to follow the greedy approach that we proved to fail even for simple structures in hard-to-manipulate instances.

% In what follows, we  assume, for sake of presentation, that the expected influence $\Expec_H\left[\chi(S, H)\right]$ can be computed in poly-time. However, if this is not the case, we can still use a Monte Carlo simulation to approximate the expected influence within a factor $\gamma$, for every $\gamma > 0$. It turns out that, using such an approximation in place of the correct value for $\Expec_H\left[\chi(S, H)\right]$ will alter the approximation ratio of the proposed algorithms only for an additive factor $\varepsilon = \varepsilon(\gamma)$ as discussed by~\citeauthor{kempe2003maximizing}~[\citeyear{kempe2003maximizing}].
\begin{theorem}
\label{thm:approx}
Let $\delta \le B$. There is a greedy poly-time  algorithm returning a $\rho$-approximation to the ECS problem, with
 $$
  \rho = \frac{B - \delta + 1}{2\,\delta\, B}\left(1-\frac{1}{e}\right).
 $$
\end{theorem}

\begin{proof}
Let $m^*$ with $|m^*|=\delta$ be the message that cause each voter to vote for $c_0$, whatever was the ranking before the reception of this message. 
Our algorithm selects $\left\lfloor\frac{B}{\delta}\right\rfloor$ seeds through the greedy algorithm to maximize $\Expec_H[\chi(S,H)]$ (\emph{i.e.}, take at each time the seed that most increases this quantity), and let each of them to send the message $m^*$.
It directly follows that this algorithm runs in poly-time in greedy fashion.

In order to formally prove the approximation factor of this algorithm for the election control problem, let us denote
with $\hat{S}$ the set of seeds returned by our algorithm,
with $S^*$ the set of seeds maximizing $\Expec_H[\Delta_\MoV(S,M^*,H)]$,
with $S'$ the set of seeds of size $B$ that maximizes $\Expec_H[\chi(S,H)]$ and
with $S''$ the set of seeds of size $k = \left\lfloor\frac{B}{\delta}\right\rfloor$ that maximizes $\Expec_H[\chi(S,H)]$.

It is known that the function $\Expec_H[\chi(S, H)]$ is monotone and submodular on $S$
\citep{kempe2003maximizing}, \emph{i.e.}, $\Expec_H[\chi(S,H)] \leq \Expec_H[\chi(T,H)]$ and $\Expec_H[\chi(S \cup \{x\}, H)] - \Expec_H[\chi(S, H)] \geq \Expec_H[\chi(T \cup \{x\}, H)] - \Expec_H[\chi(T, H)]$ for every $S \subseteq T$ and every $x \notin T$. Consequently, the greedy algorithm is known to return, for every $k$, a set of $k$ seeds whose influence is an $\left(1 - \frac{1}{e}\right)$-approximation of the maximum expected influence achievable with $k$ seeds \citep{kempe2003maximizing}. 
Hence, we have that:
\begin{equation}
 \label{eq:proof1}
 \Expec_H[\chi(\hat{S},H)] \geq \left(1 - \frac{1}{e}\right) \Expec_H[\chi(S'',H)].
\end{equation}

 Note that $|V_{c}|- \Expec_H\left[\left|V^*_c(S^*,M^*,H)\right|\right] \le  \Expec_H\left[\chi(S^*,H)\right]$ for every $c \neq c_0$, since at most one vote can be lost by $c$ for every influenced node in graph $H$.
 Then we have that
\begin{equation}
	\begin{aligned}
	\label{eq:proof2}
 	& \max_{c \neq c_0} |V_{c}|- \Expec_H\left[\max_{c \neq c_0} \left|V^*_c(S^*,M^*,H)\right|\right]\\ 
 	& \qquad \le \max_{c \neq c_0} \Big\{|V_{c}|- \Expec_H\left[\left|V^*_c(S^*,M^*,H)\right|\right] \Big\}\\
 	& \qquad \le \Expec_H\left[\chi(S^*,H)\right].
	\end{aligned}
\end{equation}
A similar argument proves that
\begin{equation}
 \label{eq:proof3}
 \Expec_H\left[\left|V^*_{c_0}(S^*,M^*,H)\right|\right] - |V_{c_0}| \le \Expec_H\left[\chi(S^*,H)\right].
\end{equation}

 Moreover, by submodularity of $\chi$, it holds that $\frac{\Expec_H\left[\chi(S',H)\right]}{|S'|} \leq \frac{\Expec_H\left[\chi(S'',H)\right]}{|S''|}$.
Hence, since $|S'| = B$ and $|S''| = \left\lfloor\frac{B}{\delta}\right\rfloor \geq \frac{B-\delta+1}{\delta}$, we achieve that
\begin{equation}
\label{eq:proof4}
 \Expec_H\left[\chi(S',H)\right] \leq \frac{\delta B}{B-\delta+1}\Expec_H\left[\chi(S'',H)\right].
\end{equation}

 Moreover, by definition of $\Delta_\MoV^S$, we have 
 \begin{multline*}
 \Expec_H[\Delta_\MoV^S(S^*,M^*,H)] =  \\ \Expec_H\left[\left|V^*_{c_0}(S^*,M^*,H)\right|\right] - \Expec_H\left[\max_{c \neq c_0} \left|V^*_c(S^*,M^*,H)\right|\right] -  \left(|V_{c_0}| - \max_{c \neq c_0} |V_{c}| \right).
 \end{multline*}
 Hence, we directly achieve that
 \begin{multline*}
 	\Expec_H[\Delta_\MoV^S(S^*,M^*,H)] = \\ \left(\Expec_H\left[\left|V^*_{c_0}(S^*,M^*,H)\right|\right] - |V_{c_0}|\right) + \left( \max_{c \neq c_0} |V_{c}| - \Expec_H[\max_{c \neq c_0} |V^*_c(S^*,M^*,H)|] \right).
 \end{multline*}
Then, have that
%\begin{equation}
\begin{align*}
  \Expec_H[\Delta_\MoV^S(S^*,M^*,H)] &  %\nonumber\\
%   & = \Expec_H\left[\left|V^*_{c_0}(S^*,M^*,H)\right|\right] - \Expec_H\left[\max_{c \neq c_0} \left|V^*_c(S^*,M^*,H)\right|\right] - \left(|V_{c_0}| - \max_{c \neq c_0} |V_{c}| \right) \nonumber\\
%   & = \left(\Expec_H\left[\left|V^*_{c_0}(S^*,M^*,H)\right|\right] - |V_{c_0}|\right) + \left(  \max_{c \neq c_0} |V_{c}| - \Expec_H\left[ \max_{c \neq c_0} \left|V^*_c(S^*,M^*,H)\right|\right] \right) \nonumber\\
%  & 
 \leq 2\,\Expec_H\left[\chi(S^*,H)\right]  \tag*{\text{(by \eqref{eq:proof2} and \eqref{eq:proof3})}}\\
& \leq 2\,\Expec_H\left[\chi(S',H)\right] \tag*{\text{(by definition of $S'$)}}\\
 & \leq \frac{2\,\delta\, B}{B-\delta+1}\Expec_H\left[\chi(S'',H)\right] \tag*{\text{(by \eqref{eq:proof4})}}\\
 & \leq \frac{2\,\delta\, B}{B-\delta+1} \left(1-\frac{1}{e}\right)^{-1}\Expec_H\left[\chi(\hat{S},H)\right] \tag* {\text{(by \eqref{eq:proof1})}}\\
  & \leq \frac{2\,\delta\, B}{B-\delta+1} \left(1-\frac{1}{e}\right)^{-1}\Expec_H\left[\Delta_\MoV^S(\hat{S},M^*,H)\right] \nonumber,
\end{align*}
 %\end{equation}
 where the last inequality follows from the fact that, by definition of $m^*$, all the influenced nodes will vote for $c_0$.\end{proof}

Notice that Theorem \ref{thm:approx} guarantees a constant approximation factor whenever $\delta$ is fixed. 
Our algorithm heavily depends on the possibility that the seeds can send messages with information on multiple candidates. Indeed, Theorem~\ref{thm:DkS} shows that, without this possibility, it is unlikely that the ECS problem is approximable within a constant factor even with two candidates.

% \subsection{Special Cases}
% We next show some specific results for the special ranking revision functions discussed in Section~\ref{sec:model}. 
% %We find instructive to distinguish the case of three candidates from the case of four or more candidates. 
% %\subsubsection{Three Candidates}
% \input{costantBound}
% 
% %\subsubsection{Four Candidates or More}
% \input{fourPlayer}

\subsection{Seeding Complexity in Variants of the Model}
\label{subsec:extensions}
We describe some extensions and variants of our model and  show how most of the results presented in the previous section extend to these settings.
\subsubsection{Bribed Seeds}
In our model, seeds act as initiators of positive and/or negative messages about the candidates. However, apart from that, their behavior is exactly the same as any other node in the network. In particular, the messages that they receive affect their preference ranks and, consequently, their vote.
We also study a variant, in which seeds are \emph{bribed} (\emph{i.e.}, for each seed, their preferred candidate is set by the manipulator, and it is independent from her initial preference rank, and from the messages that she sends and receives).

It directly follows that the reduction described in the proof of Theorem~\ref{thm:inapprox} does not work in this variant.
However, we next show that, even in this variant, the \textsc{ECS} problem is inapproximable.
\begin{theorem}
\label{thm:bribed_inapprox}
Given the set of ECS instances said hard to manipulate and with at least three candidates, for any $\rho > 0$, there is not any poly-time algorithm returning a $\rho$-approximation to the ECS problem with bribed seeds, unless $\mathsf{P} = \mathsf{NP}$.
\end{theorem}
\begin{proof}
 Consider the reduction described in the proof of Theorem~\ref{thm:inapprox}, except that now each node is enlarged into a clique of size $(h+1) \,\rho'$, where $\rho' > \rho$.
 Hence, if a set cover of size at most $h$ exists, then, $\Delta_{\MoV}(S^*,I^*,H) \geq (h+1) \, \rho'$, otherwise the only nodes that eventually change opinion are the seeds, that are at most $h+1$. Thus any $\rho$-approximate algorithm must be able to distinguish these two cases and thus solves the \textsc{Set-Cover} problem in polynomial time.
\end{proof}
Instead, it is easy to check that Theorem~\ref{thm:approx} is unaffected by the fact that seeds are bribed, and thus a constant approximation is still possible when $\delta \leq B$.

\subsubsection{Other Objective Functions}
In addition to the maximization of the increase in the margin of victory, alternative objective functions, previously studied in~\cite{wilder2018controlling}, may be of interest.

For example, one may want to maximize the increase in the probability of victory. For this objective function, it is not trivial to see that Theorem~\ref{thm:inapprox} keeps holding.
However, notice that this objective function makes the problem even harder than maximizing the increase in the margin of victory. Indeed, for the latter objective, Theorem~\ref{thm:approx} implies that a $\frac{1}{2}\left(1 - \frac{1}{e}\right)$-approximation can be computed in poly-time when only two candidates are involved. It is instead not hard to see that, to maximize the increase in the probability of victory with only two candidates, it is sufficient that all selected seeds send the same message. Hence, for two candidates, maximizing the increase in the probability of victory in our setting is the same as doing it in the setting studied in~\cite{wilder2018controlling}. Hence, the problem cannot be approximated, unless $\mathsf{P}=\mathsf{NP}$, within a factor $\rho > 0$, even for two only candidates.

An apparently weaker goal would be that one of computing the set of seeds and the corresponding messages so that the probability of victory  is merely above a given threshold (so the set of feasible solutions would be larger than in the setting described above). Unfortunately, this objective function does not make the problem easier to be solved. Indeed, not only Theorem~\ref{thm:inapprox} holds in this setting regardless of the threshold, but one may show that, as for the goal of maximizing the probability of victory, the inapproximability still holds when only two candidates are available \cite{wilder2018controlling}.
 
\subsubsection{Threshold Dynamics}
The results provided in Section~\ref{sec:general}, that are based on a multi-issue independent cascade model, can be extended to settings in which information diffuses according to the \emph{linear threshold model} \cite{kempe2003maximizing}. This  represents the most prominent among the diffusion models alternative to the independent cascade. In the linear threshold model, 
for each node $v$ of the network, there is a threshold $\theta_v$ drawn randomly in $[0,1]$, and incoming edges $(u,v)$ have a weight $w_{u,v}$ such that $\sum_{(u,v)} w_{u,v} = 1$.
Then, a node $v$ becomes active at time $t$ only if the sum of weights of edges coming from active nodes passes the threshold.

It is known that this diffusion model leads to different dynamics with respect to the independent cascade model. Still, we show that our proofs can be adapted. In particular,  Theorem~\ref{thm:inapprox} and Theorem~\ref{thm:approx} still hold.

% \textcolor{red}{Non mi torna il fatto delle $r$ copie. Oltretutto, gia senza introdurre le r copie mostri l'inapprosimabilità}
Specifically, for the inapproximability result, we use, in place of \textsc{Set-Cover}, a reduction from \textsc{Vertex-Cover}.
This is the problem of deciding whether, given a graph $Z$ of $g$ nodes and an integer $h$, there is a subset $S$ of at most $h$ nodes of $Z$ such that every edge of $Z$ has at least one endpoint in $S$.
The reduction  is similar to the one described in Theorem~\ref{thm:inapprox}. Namely, the component $G_1$ consists of $r$ copies of the graph $G$. Now by setting $n=r\,(g-h)$, we let components $G_2$ and $G_3$ to have the same number of nodes as in the proof of Theorem~\ref{thm:inapprox}, except that now the nodes in each component are not arranged as a clique, but as a directed ring (so that a message sent by a node in one component will activate all nodes in that component regardless of their threshold). Notice that, by considering the same initial ranks as in the proof of Theorem~\ref{thm:inapprox}, the expected margin of victory of $c_0$ increases by $1$ if and only if there is in $G$ a vertex cover of size at most $h$. Finally, the expected increase in the margin of victory when no vertex cover exists can be made as low as desired by increasing~$r$.

On the other side, it directly follows that the greedy algorithm proposed in Theorem~\ref{thm:approx} works, with the same approximation factor, even with the linear threshold diffusion model. Indeed, it is known that the influence maximization is a monotonic and submodular function even with this dynamics \citep{kempe2003maximizing}, and it can be observed that this is sufficient to make the proof of Theorem~\ref{thm:approx} to hold.

\subsubsection{Seeds with Different Costs}
In our model, we assume that each node can be selected as a seed at the same cost. This can be highly unrealistic. Hence, an extension to our model would be to assume that each node $u$ has a different cost $w(u)$ that should be paid for each message initiated by that node.

Intuitively, this extension makes the election control problem harder.
Hence, inapproximability results clearly extend to this setting too.
Actually, we can prove that the inapproximability holds even if we restrict to undirected graphs. We can do that by adapting, within the framework of the proof of Theorem~\ref{thm:inapprox},
the reduction described in~\cite{khanna2014influence} for the hardness of the influence maximization problem, from the problem of finding a vertex cover in a \emph{cubic graph}, \emph{i.e.}, a graph in which every vertex has degree exactly equal to three.~\footnote{Note that, in that reduction, the author assumes that there is a set of non-allowed seed, that can be simulated in our setting by setting the initial preference rank for those nodes such that $c_0$ is the best ranked candidate and $c_2$ the last ranked one.}

Surprisingly, however, we have that, whenever $\delta \leq B$, a poly-time algorithm returning a constant approximation to the election control problem exists even if the nodes have heterogeneous seeding costs.
Indeed, this setting admits a poly-time algorithm for influence maximization returning a $\left(1 - \frac{1}{\sqrt{e}}\right)$-approximation of the optimal seed set \citep{nguyen2013budgeted}.
Then, the arguments of the proof of Theorem~\ref{thm:approx} immediately prove that this algorithm
provides a $\rho$-approximation for the extension of the election control problem to voters with different costs, where
 $\rho = \frac{B - \delta + 1}{2\delta B}\left(1-\frac{1}{\sqrt{e}}\right)$.

\section{Edge Removal Complexity}
\label{sec:edgeremoval}

We study, in this section, the \textsc{ECER} and \textsc{IMER} problems.
% \textcolor{orange}{
All the results provided in this section hold even with unitary score distances.
% }
%
%Our results on the \textsc{ECER} problem are summarized in Table~\ref{table:er} and show that, except for the trivial case with two candidates, unlimited budget, and a single message, the problem is hard even to approximate unless  \Poly~$=$~\NP.
% Most of the proofs in this section are only sketched.
%
%\begin{table}[htb]
%	\resizebox{\columnwidth}{!}{%
%		\begin{tabular}{|c|c|c|c|}
%			\cline{1-4}
%			& \multicolumn{2}{c|}{\bf single message}  & \bf multiple messages, \\  \cline{2-3}
%			&\bf two candidates& \bf three candidates &   \bf two or more candidates\\\cline{1-4}
%			\bf limited budget &   $\notin$ \APX~ (Cor~\ref{cor:ERFixed}) & $\notin$ \APX ~ (Thm~\ref{thm:ER_single_message})& $\notin$ Exp-\APX ~ (Thm~\ref{thm:ER_multiple_messages}) \\\cline{1-4}
%			\bf unlimited budget &\Poly~(Obs~\ref{obs:polyECER})  & $\notin$ \APX ~ (Thm~\ref{thm:ER_single_message}) & $\notin$ Exp-\APX ~ (Thm~\ref{thm:ER_multiple_messages}) \\ \cline{1-4}
%	\end{tabular}}
%	\caption{Complexity of election control by edge removal.}
%\label{table:er}
%\end{table}
%
Initially, we focus on the IMER problem when one can only remove edges, as its characterization is useful for the characterization of the \textsc{ECER} problem with two candidates and limited budget.
We show that the IMER problem is hard.
Our proof reduces from the \textsc{Maximum-Subset-Intersection} problem that does not admit any constant-factor approximation polynomial-time algorithm unless \Poly~$=$~\NP, as showed in~\cite{SHIEH2012723}.

\begin{definition}[\textsc{Maximum-Subset-Intersection} (\textsc{MSI})]
	Given a finite set $N = \{z_1,\dots,z_n\}$ of elements, a collection $X= \{x_1,\dots,x_g\}$ of sets with $x_i \subset N$, and a positive integer $h$, the goal is to find exactly $h$ subsets $x_{j_1},\dots,x_{j_h}$ whose intersection size $|x_{j_1} \cap \ldots \cap x_{j_h}|$ is maximum.
\end{definition}

\begin{theorem}\label{thm:hardIinfMin} For any constant $\rho>0$, there is no polynomial-time algorithm returning a $\rho$-approximation to \textsc{IMER} problem when the budget $B$ is finite, unless \Poly~$=$~\NP.
\end{theorem}
\begin{proof}
	We reduce from \textsc{MSI}, showing that a constant-factor approximation algorithm for IMER implies the existence of a constant-factor approximation for \textsc{MSI}, thus having a contradiction unless \Poly~$=$~\NP.
	Given an instance $(X,N)$ of \textsc{MSI}, we build an instance of \textsc{IMER} as follows.
	For each element $z_i$, we add $n^2\,g^2$ nodes $v_{z_i,j}$ with $j \in \{1,\ldots,n^2g^2\}$.
	For each set $x_i \in X$, we add two nodes $v_{x_i,1}, v_{x_i,2}$ and an edge from $v_{x_i,1}$ to $v_{x_i,2}$. All $v_{x_i,1}$ are seeds, while each $v_{x_i,2}$ has an edge to each node $v_{z_i,j}, z_i \in N \setminus x_i$ with $ j \in \{1,\ldots,n^2g^2\}$, \emph{i.e.}, all the nodes of all the elements not in the set $x_i$.
	Figure~\ref{fig:ER_IM} depicts an example of network built with the above mapping.
	The budget is set equal to $g-h$.
	\begin{figure}[ht]
		\centering
		\includegraphics[width=0.7\linewidth]{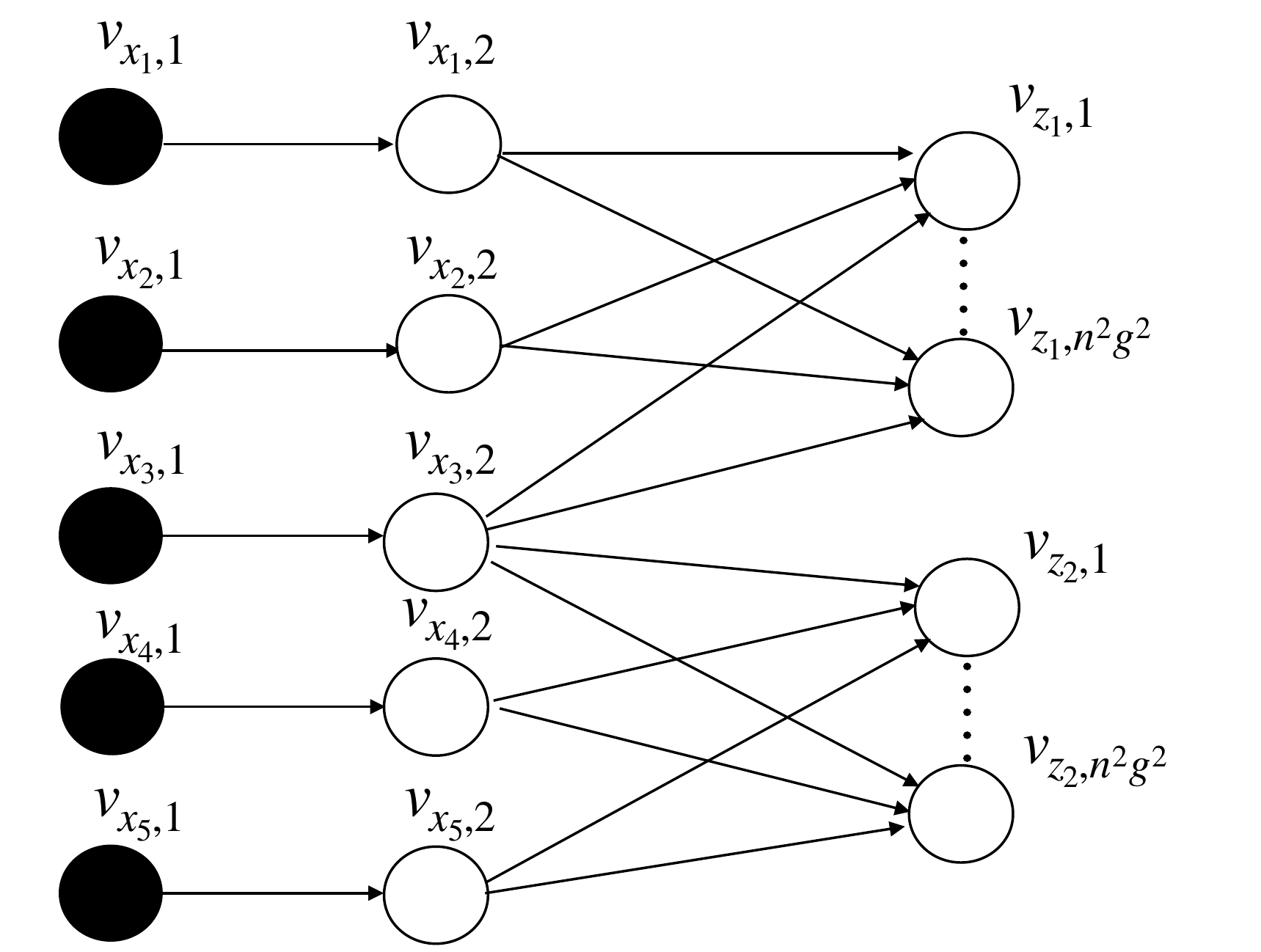}
		\caption{Structure of the election control problem used in the proof of Theorem \ref{thm:hardIinfMin}.}
		\label{fig:ER_IM}
	\end{figure}
	
	Notice that, in the optimal solution, only edges from nodes $v_{x_i,1}$ to $v_{x_i,2}$ are removed.
	Thus, the problem reduces to choose $g-h$ sets $x_i \in X$ and remove the edges from  $v_{x_i,1}$ to $v_{x_i,2}$, such that as least as possible nodes $v_{z_i,j}$ are influenced.
	The optimal value is obtained by choosing $X^*\subset X$ of cardinality $h$ that is solution of \textsc{MSI} and then removing the edges from  $v_{x_i,1}$ to $v_{x_i,2}$ for all $x \in X \setminus X^*$. 
	Call this set of edges $E^*$.
	If we remove edges $E^*$, all the nodes $v_{z_i,j}, z_i \in \cap_{x_i \in X^*} x_i$ are not influenced, since there are no edges from $v_{x_i,2}, x \in X^*$ to $v_{z_i,j}$, as they all exist in the complement of the bipartite graph.
	
	The relationship between the optimal solution of IMER and that one of \textsc{MSI} is $\Delta I^-(E^*)=g-h+OPT \,n^2\,g^2$, where $OPT$ is the optimal solution to \textsc{MSI}.
	Assume by contradiction there exists an $\rho$-approximation algorithm $\mathcal A$ for IMER, where $\rho \in (0,1)$.
	This implies that there exists an edge set $E'$ such that $\Delta I^-(E')=g-h+APX \,n^2\,g^2$, where $APX$ is an approximation of \textsc{MSI}.
	% 	~\footnote{If the solution removes some edges between $v_{x_i,2}$ and $v_{z_i,j}$, it is easy to find a better solution that removes only edges from $v_{x_i,1}$ to $v_{x_i,2}$.}
	%
	Since $\Delta I^-(E') \ge \rho \,\Delta I^-(E^*)$, then $g-h+APX \,n^2\,g^2 \ge (g-h+OPT \,n^2\,g^2) {\rho}$, and 
	$APX \ge \frac{(g-h)(\rho-1)}{n^2\,g^2}+\rho \,OPT$.
	Hence, there exists a $\rho'$ such that $APX \ge \rho' OPT$ and an algorithm $\mathcal A'$ for \textsc{MSI} with a $\rho'$-approximation factor.
\end{proof}
% \begin{proof}[Proof Sketch]
% Given an instance $(X,N)$ of \textsc{MSI}, we build an instance of \textsc{IMER} as depicted in Figure~\ref{fig:ER_IM}
% 	The budget is $m-h$.
% Notice that
% 	the optimal value is obtained choosing $X^*\subset X$ of cardinality $h$ solution of \textsc{MSI} and removing the edges from  $v_{x_i,1}$ to $v_{x_i,2}$ for all $x \in X \setminus X^*$. 
% 	Call this set of edges $E^*$.
% 	The relationship between the optimal solution of influence maximization and \textsc{MSI} is $\Delta I^-(E^*)=m-h+OPT n^2m^2$, where $OPT$ is the optimal solution to \textsc{MSI}.
% 
% 	Suppose there exists an $\rho$-approximation algorithm $A$ for influence minimization by edge removal.
% 	This implies that there exists a edge set $E'$ such that $\Delta I^-(E')=m-h+APX n^2m^2$, where APX is an approximation of \textsc{MSI}.
% 	Since $\Delta I^-(E') \ge (m-h+OPT n^2m^2) {\rho}$, and 
% 	$APX \ge \frac{(m-h)(\rho-1)}{n^2m^2}+\rho OPT$,
% 	there exists a $\rho'$ such that $APX \ge \rho' OPT$ and an algorithm $A'$ for \textsc{MSI} with a $\rho'$ approximation factor.
% \end{proof}
 
We can state the following corollary, whose proof directly follows from the proof of the above theorem. 
\begin{corollary}\label{cor:ERFixed}
For any constant $\rho>0$, there is not any polynomial time algorithm returning a $\rho$-approximation to the \textsc{ECER} problem when budget $B$ is finite even when there are two candidates and single-news-article-messages, unless \Poly~$=$~\NP.
\end{corollary}
\begin{proof} We can build an instance of ECER with the same graph of Theorem \ref{thm:hardIinfMin}, two candidates, all nodes with scores $ \langle 1,0\rangle $ and seeds with messages $(-1,0)$. It is easy to see that $\Delta_\MoV= 2 \Delta I^-$. Since approximating $\Delta I^-$ is hard, it follows that approximating $\Delta_\MoV^-$ is hard too.%\hfill$\Box$  
\end{proof}

Since with finite budget even the setting with single-news-article messages and two candidates is hard, we focus on those problems in which the budget is unlimited ($B = \infty$). Notice that, while a finite budget corresponds to the case in which a manipulator pays a platform, in the case in which the manipulator is the platform itself, the budget is actually unlimited. 

In networks with single-news-article-messages and only two candidates, the optimal solution can be found easily. 
Intuitively, the problem becomes easy because we can easily solve \textsc{IMER}. If we have unlimited budget, the optimal solution to \textsc{IMER} removes all the edges.
% 
% \begin{theorem}
% 	There exists a polynomial-time algorithm for the IMER problem with unlimited budget.
% \end{theorem}
% %
It is then easy to extend this solution to solve the ECEA when all the seeds send the same message on the same single candidate  and there are only two candidates:
if the message is negative for $c_0$, \emph{e.g.}, $q_0=-1$ or $q_1=1$, we remove all edges from the network, clearly minimizing the negative effects of the diffusion of the message;
if the message is positive for $c_0$, \emph{e.g.}, $q_0=1$ or $q_1=-1$, since we cannot increase the diffusion by removing edges, we do not modify the network.
From the previous arguments, we can directly state the following.
\begin{obs}
\label{obs:polyECER}
	There exists a polynomial-time algorithm for the \textsc{ECER} problem with single-news-article messages, two candidates, and unlimited budget.
\end{obs}

Now we show that extending the setting to three or more candidates elections or to the diffusion of different messages makes the problem hard.
We introduce the \textsc{Independent-Set} problem, that is known not to be approximable to any constant factor in~\cite{v003a006}, to prove the hardness of the \textsc{ECER} when there are three candidates and messages are single-news-article.

\begin{definition}[\textsc{Independent-Set}]
	Given a graph $G=(X,N)$, with $|X|=g$ vertexes and $|N|=n$ edges,  find the largest set of vertexes  $X^*$ such that there is no edge connecting two vertexes in $X^*$.
\end{definition}

\begin{theorem} \label{thm:ER_single_message}
	For any constant $\rho > 0$, there is no polynomial time algorithm returning a $\rho$-approximation to the \textsc{ECER} with single-news-article messages even when there are three candidates and the budget is unlimited, unless \Poly~$=$~\NP.
\end{theorem}
\begin{proof}
	Given an instance of \textsc{Independent-Set}, we build an instance of election control as follows.
	We add a line $L_1$ of $n\,g-g$ nodes with preference $ \langle 2,0,1\rangle $ and we seed the first node of the line with a message with $q_0=q_1=0$ and $q_2=1$.
	We add a node $v_{x_i}$ for each node $x_i \in X$ with preferences $ \langle 2,0,1\rangle $ and an edge from the last element of the line $L_1$ to $v_{x_i}$.
	For each element $z_i \in N$, we add a line $L_{z_i}$ of $g$ nodes with preferences $ \langle 0,2,1\rangle $ and an edge from each $x_j \ni z_i$ to the first node of $L_{z_i}$.
	Moreover, we add $n^2\,g^2$ isolated nodes with preferences $ \langle 2,1,0\rangle $ and $n^2\,g^2$ isolated nodes with preferences $ \langle 1,2,0\rangle $.
	Figure~\ref{fig:ER_single_message} depicts an example of network produced with the above mapping.
	Note that, if no edge is removed, all non-isolated voters change their preferences and vote $c_2$, implying $\MoV(S,M,E,H) = 0$.
	We prove that a constant-factor approximation for ECER would lead to a constant-factor approximation for \textsc{Independent-Set}.
	
	\begin{figure}[htb]
		\centering
		\includegraphics[width=1\linewidth]{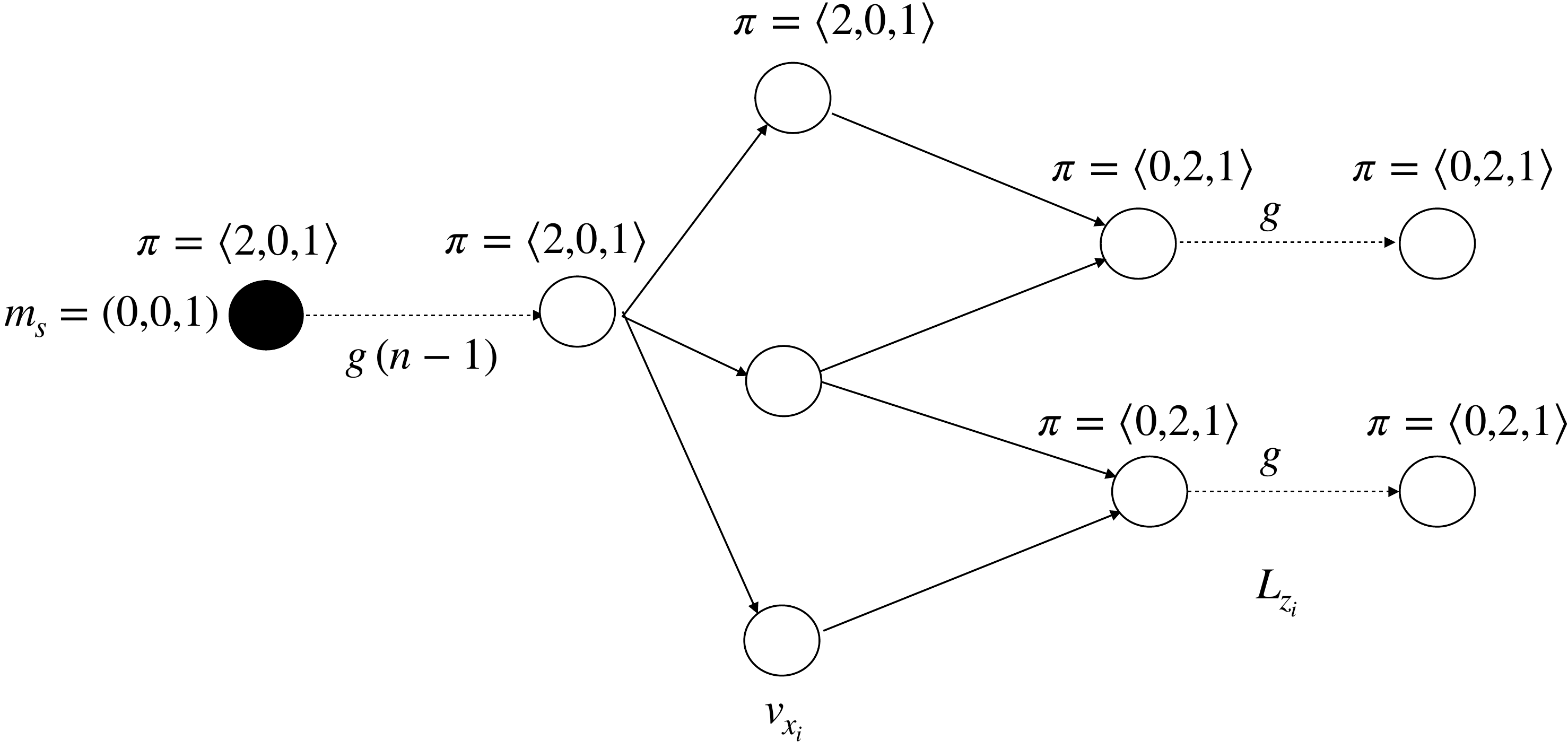}
		\caption{Structure of the election control problem used in the proof of  Theorem \ref{thm:ER_single_message}.}
		\label{fig:ER_single_message}
	\end{figure}
	
	Suppose that there exists a set of edges $E'$, such that $\Delta_\MoV^-(E') > 0$. Then $c_1$ looses all her votes in non-isolated nodes, otherwise $V_{c_0}(S,M,E',H)^*\le n^2\,g^2--(n-1)g$ and $V_{c_1}^*\ge n^2\,g^2-(n-1)g$.
	This suggests that the optimal solution is given by the greatest independent set $X^*\subseteq X$. In particular, $E^*$ is given by all the edges from the last node of $L_1$ to all $v_{x_i}$ with $x_i \in X^*$.  
	Notice that the set of active nodes $v_{x_i}, x_i \in X \setminus X^*$ is the complement of a maximum independent set and hence a minimum vertex cover.
	Thus, removing all edges in $E^*$, we obtain $\Delta_\MoV^-(E^*) = | X^*|$.
	
	Suppose there exists a $\rho$-approximation algorithm $\mathcal{A}$ for the ECER problem that removes edges $E'$. This implies that $\Delta_\MoV^-(E') \ge \rho \,\Delta_\MoV^-(E^*)$, where $E'$ is the set of the edges removed by algorithm $\mathcal A$.
	Since $\Delta_\MoV^-(E')>0$, $	\mathcal A$ removes only edges from $L_1$ to $v_{x_i}$ since, if it removes edges between nodes in $L_1$, we would have $\Delta_\MoV^-(E')\le 0$.
	% 	~\footnote{If $\mathcal A$ also removes edges from nodes $v_{x_i}$ to lines $L_{z_i}$, we can easily find a better solution that removes the edge from $L_1$ to $v_{x_i}$.}
	%
	Moreover, all lines $L_{z_i}$ must be active. Hence, the active vertexes $v_{x_i}$ are a vertex cover and the inactive vertexes in $v_{x_i}$ are an independent set. 
	We remark that the value of $\Delta_\MoV^-(E')$ is exactly the number of inactive vertexes, \emph{i.e.}, the vertexes in the independent set. Thus, if there exists a $\rho$-approximation algorithm for ECER, there exists a $\rho$ approximation algorithm for \textsc{Independent-Set}, leading to a contradiction.
\end{proof}
% \begin{proof}[Proof Sketch]
% 	Given an instance of \textsc{Independent-Set}, we build an instance of election control as depicted in Figure~\ref{fig:ER_single_message}.
% 	We prove that a constant factor for election manipulation would imply a constant factor approximation for \textsc{Independent-Set}.
% 	
% 	Suppose that there exists a set of edges $E'$, such that $\Delta_\MoV^-(E') > 0$. This implies that $c_1$ looses all her votes in not isolated nodes.
% 	The optimal solution is given by the smallest independent set $X^*\subseteq X$.
% 	Notice that the set of active nodes $v_{x_i}, x_i \in X \setminus X^*$ is the complement of a maximum independent set and hence a minimum vertex cover.
% 	Thus we obtain $\Delta_\MoV^-(E^*) = | X^*|$.
% 	
% 	Suppose there exists a $\rho$-approximation algorithm $A$ for the election control problem that removes edges $E'$.
% 	Since $\Delta_\MoV^-(E')>0$, $A$ removes only edges from $L_1$ to $v_{x_i}$.
% 	Moreover, all lines $L_{z_i}$ must be active. Hence, the active vertexes $v_{x_i}$ are a vertex cover and the inactive vertexes in $v_{x_i}$ are an independent set. 
% 	We remark that the value of $\Delta_\MoV^-(E')$ is exactly
% 	the number of 	vertexes in the independent set. Thus, if there exists a $\rho$-approximation algorithm for the election control problem, there exists a $\rho$ approximation algorithm for \textsc{Independent-Set}.
% \end{proof}

We now focus on the case in which messages can be arbitrary and only two candidates. We reduce from the \textsc{Set-Cover} problem to prove the hardness of the \textsc{ECER} problem even in these settings.

\begin{theorem} \label{thm:ER_multiple_messages}
	For any $\rho > 0$, there is no polynomial time algorithm returning a $\rho$-approximation to the \textsc{ECER} even with two candidates and unlimited budget, unless \Poly~$=$~\NP.
\end{theorem}
\begin{proof}
	Consider an instance of \textsc{Set-Cover}. We suppose, w.l.o.g., $n > g$ and build a graph as follows.
	We add a node $v_1$ with preferences $ \langle 1,0\rangle $ and seeded with messages $q_0=1$ and $q_1=-1$,
	a node $v_2$ with preferences $ \langle 1,0\rangle $ and seeded with message $q_0=-1$, and an edge between $v_1$ and $v_2$.
	%
	%We add a node $v^3$ with preference $c_0>c_1$ and seeded with $(-,+)$, and an edge with probability $\frac{1}{2}$ between $v^1$ and $v_3$.
	%
	We add a line $L_1$ of $n^2-h-1$ nodes with preferences $ \langle 1,0\rangle $ and an edge of probability $\frac{1}{2}$ from $v_1$ to the first node of the line and an edge from $v_2$ to the first node of the line.
	Moreover we seed the first node of the line with message $q_1=1$.
	We add a node $v_{x_i}$ for each set $x_i \in X$ with preferences $ \langle 1,0\rangle $ and an edge from the last element of the line $L_1$ to $v_{x_i}$.
	For each element $z_i \in N$, we add a line $L_{z_i}$ of $n$ nodes with preference $ \langle 0,1\rangle $ and an edge from each $x_j \in z_i$ to the first node of $L_{z_i}$.
	Moreover, we add $g-h+1$ isolated nodes with preferences $ \langle 0,1\rangle $.
	Figure~\ref{fig:ER_multiple_messages} depicts an example of network produced with the above mapping.
	Note that, if no edge is removed, all the voters do not change their preferences and $\MoV = 0$.
	We prove that $\Delta_\MoV^-$ is larger than $0$ if and only if \textsc{Set-Cover} is satisfiable.
	\begin{figure}[htb]
		\centering
		\includegraphics[width=1\linewidth]{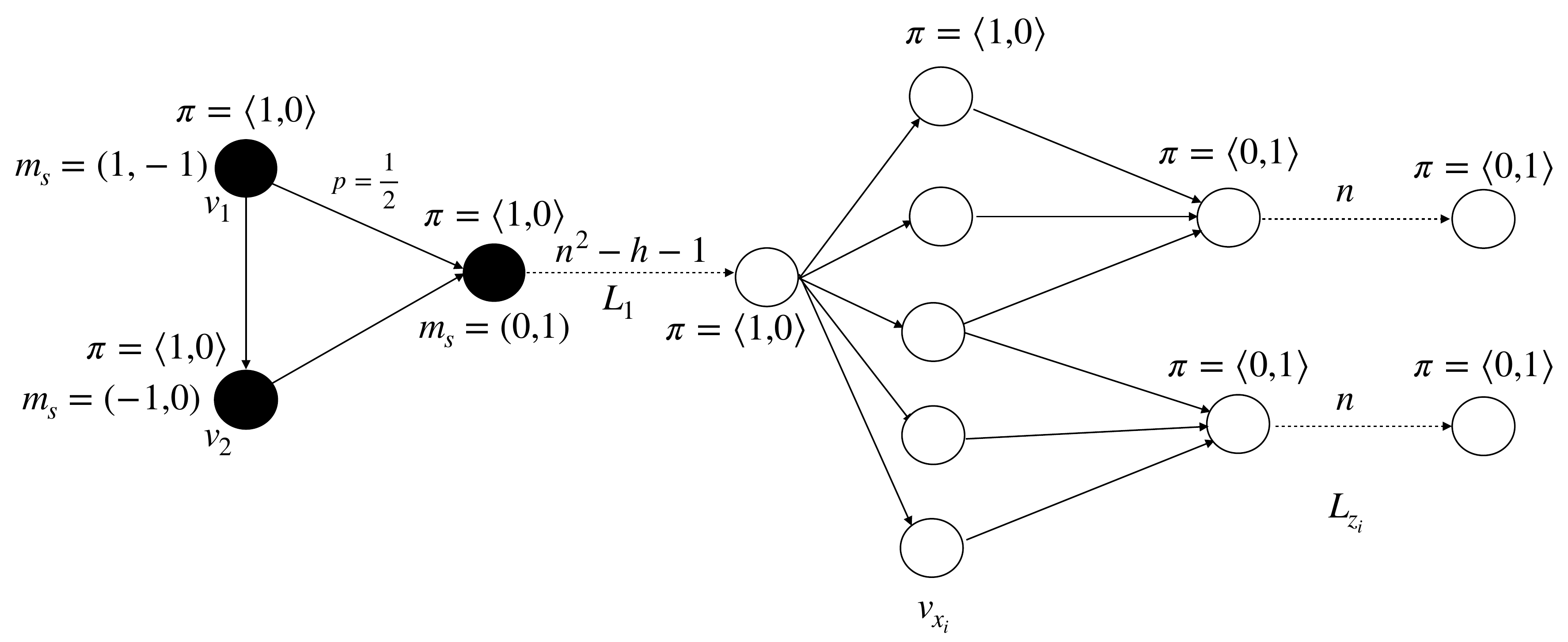}
		\caption{Structure of the election control problem used in the proof of  Theorem \ref{thm:ER_multiple_messages}.}
		\label{fig:ER_multiple_messages}
	\end{figure}

	\textbf{If}.
	Define the set of removed edges $E^*$ as composed by the edge between $v_2$ and $L_1$ and the incoming edge of each $v_{x_i}$ with $x_i \in X \setminus X^*$.
	We have two possible live graphs: $H_1$ if the edge between $v_1$ and $L_1$ is active, $H_2$ otherwise.
	Thus, $\Delta_\MoV^-(E^*, H_1)=2n^2$ and $\Delta_\MoV^-(E^*, H_2)=2(-n^2+h+1-h)=-2n^2+2$.
	Hence, $\Delta_\MoV^-(E^*)=1$.
	
	\textbf{Only if}.
	Suppose we do not remove neither the edge from $v_1$ towards $v_2$ nor the edge from $v_2$ towards $L_1$. In this case, no voter changes her vote from $c_1$ to $c_0$ since all the nodes that votes for $c_1$ receive messages $q_0=1$, $q_0=-1$, $q_1=1$, and $q_1=-1$.
	Thus, one of the two aforementioned edges should be removed. It easy to see that removing the edge from $v_2$ to $L_1$ is the best choice.
	Since $c_0$ must take some of the votes of $c_1$, the message in $v_1$ must reach at least some lines in $L_z$ and no edges must be removed in $L_1$.
	We have two possible live graphs: $H_1$ if the edge between $v_1$ and $L_1$ is active, $H_2$ otherwise.
	Assume by contradiction that in $H_1$ not all lines $L_z$ vote for $c_0$.
	This implies that $\Delta_\MoV^-(E^*)\le \frac{2(n(n-1))-2(n^2+h+1)}{2}<0$, where $E^*$ is the set of removed edges.
	Hence, in $H_1$, all line $L_z$ must be active and $\Delta_{\MoV}^-(E^*,H_1) =2 n^2$.
	In $H_2$, $\Delta_{\MoV}^-(E^*)$ must be larger than $-2n^2+1$ and at most $h$ nodes $v_{z_i}$ can be active.
	Thus, there exists a set cover of size $h$.
\end{proof}
% \begin{proof}[Proof Sketch]
% 	Consider an instance of \textsc{Set-Cover}. We suppose, w.l.o.g., $n > m$ and build a graph as depicted in Figure~\ref{fig:ER_multiple_messages}.
% 	We prove that $\Delta_\MoV^-$ is greater than $0$ if and only if \textsc{Set-Cover} is satisfiable.
% 	
% 	\textbf{If.}
% 	Define the set of removed edges $E^*$ as composed by the edge between $v_2$ and $L_1$ and the incoming edge of each $v_{x_i}$ with $x_i \in X \setminus X^*$.
% 	Then $\Delta_\MoV^-(E^*)=1$.
% 	
% 	\textbf{Only if.}
% 	We have two possible live graphs, $H_1$ if the edge between $v_1$ and $L_1$ is active, $H_2$ otherwise.
% 	In $H_1$, all line $L_z$ must be active and $\Delta_{\MoV}^-(E^*,H_1) =2 n^2$.
% 	In $H_2$, $\Delta_{\MoV}^-(E^*)$ must be greater than $-2n^2+1$ and at most $h$ nodes $v_{z_i}$ can be active.
% 	Thus, there exists a set cover of size $h$.
% \end{proof}

\section{Edge Addition Complexity}
\label{sec:edgeaddition}
We study, in this section, the \textsc{ECEA} and \textsc{IMEA} problems.  
% \textcolor{orange}{
All the results provided in this section hold even with unitary score distances.
% }
%
%In the following, the \emph{single-message case} refers to the situation in which there is only a single pair $(s,i)$ such that $m_s(i)\neq \cdot$, we use \emph{multi-message case} otherwise.
%
%Our results on the \textsc{ECEA} problem are summarized in Table~\ref{table:ea}. Basically, the complexity of the \textsc{ECEA} problem is the same as for \textsc{ECER}. Some proofs in this section are omitted. We refer the interested reader to the supplementary material.
%
%\begin{table}[htb]
%	\resizebox{\columnwidth}{!}{%
%		\begin{tabular}{|c|c|c|c|}
%			\cline{1-4}
%			& \multicolumn{2}{c}{\bf single message} \vline & \bf multiple messages \\  \cline{2-3}
%			&\bf two candidates& \bf three candidates &  \bf two or more candidates \\\cline{1-4}
%			\bf limited budget &   $\notin$ \APX ~ (Cor~\ref{cor:EAFixed}) & $\notin$ Exp-\APX ~ (Thm~\ref{thm:EA_single_message}) & $\notin$ Exp-\APX ~ (Thm~\ref{thm:EA_multiple_messages})\\\cline{1-4}
%			\bf unlimited budget &\Poly ~ (Obs~\ref{obs:polyECEA})  & $\notin$ Exp-\APX ~ (Thm~\ref{thm:EA_single_message}) & $\notin$ Exp-\APX ~ (Thm~\ref{thm:EA_multiple_messages}) \\ \cline{1-4}
%	\end{tabular}}
%	\caption{Complexity of election control by edge addition.}
%	\label{table:ea}
%\end{table}
%
Initially, we study the complexity of \textsc{IMEA} problem with a finite budget.
 First, we notice that the \APX-hardness of the \textsc{IMEA} problem directly follows from the \APX-hardness of the influence maximization problem by seeding. In fact, the seeding problem with network $G(V,E,p)$ and budget $B$ is equivalent to the edge-addition problem with the same graph, except for an additional isolated node $v_1$, that is the only seed, in which we can add at most $B$ edges and the probabilities $p$ of the new (added) edges are all zero except for the edges connecting $v_1$ to the nodes of $V$, whose probabilities $p$ are one.
We improve this result, showing that the IMEA problem is harder to approximate than influence maximization by seeding. Indeed, IMEA cannot be approximated to any constant factor, unless \Poly~$=$~\NP, while influence maximization by seeding can be.
In our proof, we reduce from the maximization version of \textsc{Set-Cover}, called \textsc{Max-Cover}.

\begin{definition}[\textsc{Max-Cover}]
	Given a finite set $N = \{z_1,\dots,z_n\}$ of elements, a collection $X= \{x_1,\dots,x_g\}$ of sets with $x_i \subset N$, and $h \in \mathbb{N}^+$, the objective is to select $X^* \subset X$, with $|X^*| \le h$, that maximizes $\left|\cup_{x_i \in X^*} x_i\right|$.
\end{definition}
Feige proves that deciding whether in an instance of \textsc{Max-Cover} all the elements can be covered or at most a $(1-\frac{1}{e}+\epsilon)$ fraction of them
% the elements can be covered 
is \NPHard\ for any $\epsilon>0$~\cite{Feige1998}.

\begin{theorem}\label{thm:hardIinfMax} For any constant $\rho>0$, there is no polynomial time algorithm returning a $\rho$-approximation to the \textsc{IMEA} problem when $B$ is finite, unless \Poly~$=$~\NP.
\end{theorem}
\begin{proof}
	Consider an instance of \textsc{Max-Cover}. We assume, w.l.o.g., $g<n$ and we build an instance of IMEA as follows:
	for each $i$ in $\{1,\ldots,n^8\}$, we add a node $v_i$, a node $v_{i,x_j}$ for each $x_j \in X$ and a node $v_{i,z_t}$ for each $z_t \in N$. Moreover, we add an edge from each $v_{i,x_j}$ to each $v_{i,z_t}, z_t \in x_j$ with probability $1$ and an edge from $v_{i,z_t}$ to $v_{i+1}$ with probability $1-\frac{1}{n^\frac{8}{n}}$.
	We add a node $v_{n^8+1}$ and an edge with probability $1$ towards $n^{10}$ nodes. Call the subgraph composed by these nodes $G'$.
	The resulting graph is depicted in Figure~\ref{fig:EA_IM}.
	The only seed is $v_1$, and the only edges that can be added are the edges between $v_i$ and $v_{i,x_j}, x_j \in X$ with probability one.
	The budget is $h \,n^8$.
	If \textsc{Max-Cover} is satisfiable, \emph{i.e.}, there exists a set $X^*$ that covers all the elements, there exists a solution $E^*$ to IMEA in which for each $i\in\{1,\ldots,n^8\}$ we add the edges from $v_i$ to $h$ $v_{i,x_j}$ such that if $v_i$ is active then all $v_{i,z_t}$ are active.
	In this case, $\Delta I^+(E^*)$ is larger than the expected influence on the subgraph $G'$, \emph{i.e.}, $\Delta I^+>[1-(\frac{1}{n^\frac{8}{n}})^n]^{n^8} n^{10}\ge [1-\frac{1}{n^8}]^{n^8} n^{10}>(\frac{1}{e}-\epsilon) n^{10}$ for all $\epsilon>0$ and $n$ large enough.
	Suppose each cover of size at most $h$ cover at most $\frac{3}{4}$.
	It implies that at least $\frac{n^8}{h+1}$ nodes $v_i$ have at most $h$ outwards edges and thus they leave at least $\frac{1}{4}$ vertexes $v_{i,z_t}$ without incoming edges.
	Thus the probability of activating $G'$ is smaller than $[1-(\frac{1}{{n^\frac{8}{n}}})^n]^\frac{h \,n^8}{h+1} [1-(\frac{1}{{n^\frac{8}{n}}})^\frac{3n}{4}]^\frac{n^8}{h+1} \le \frac{1}{e^\frac{h}{h+1}} \frac{1}{e^\frac{n^2}{h+1}} = \frac{1}{{e}^\frac{h+n^2}{h+1}}\le e^{-n}$.
	and $\Delta I^+(E') \le n^8 (n+g+1)+1+ e^{-n} n^{10}$.
	Clearly $\frac{\Delta I^+(E')}{\Delta I^+(E^*)} < \rho$, for each $\rho>0$ and $n$ sufficiently large.
	Hence, a $\rho$-approximation algorithm for IMEA implies that we can distinguish between satisfiable instances of \textsc{Max-Cover} and instances in which at most $\frac{3}{4}$ of the elements are covered,
	% 	reaching
	leading to a contradiction.
	% 	, unless \Poly~$=$~\NP.
\end{proof}
	\begin{figure}[htb]
\includegraphics[width=1\linewidth]{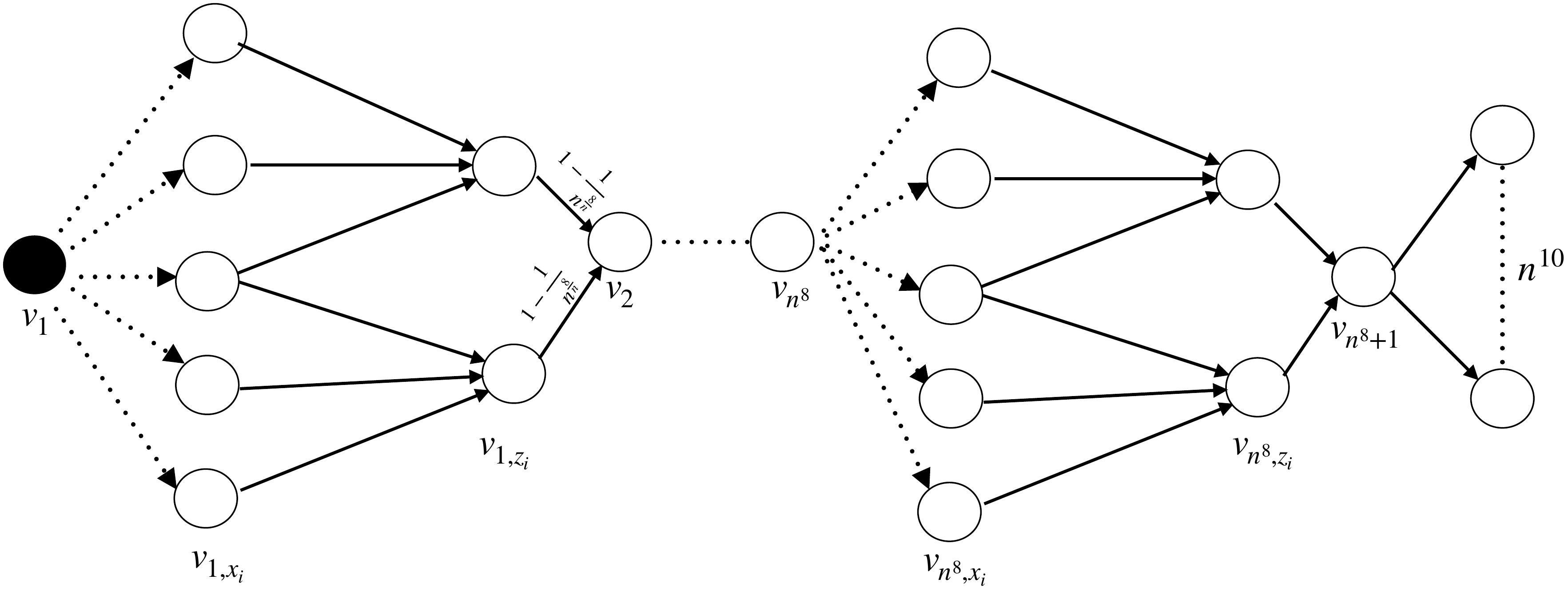}
\caption{Structure of the election control problem used in the proof of  Theorem \ref{thm:hardIinfMax}.}
\label{fig:EA_IM}
\end{figure}
	
% \begin{proof}[Proof Sketch]
% 	Consider an instance of \textsc{Max-Cover}. We assume, w.l.o.g., $m<n$ and we build an instance of influence maximization as depicted in Figure~\ref{fig:EA_IM}.
% \begin{figure}
% \includegraphics[width=1\linewidth]{images/EA_IM.pdf}
% \caption{Graph used in the reduction of Theorem \ref{thm:hardIinfMax}.}
% \label{fig:EA_IM}
% \end{figure}
% 	%
% 	The only seed is $v_1$, and the only edges that can be added are the edges between $v_i$ and $v_{i,x_j}, x_j \in M$ with probability one.
% 	The budget is $h n^8$.
% 	
% 	If \textsc{Max-Cover} is satisfiable, there exists a solution $E^*$ of influence maximization in which for each $i$ in $[1,n^8]$ we add the edges from $v_i$ to $h$ $v_{i,x_j}$ such that if $v_i$ is active then all $v_{i,z_t}$ are active.
% 	%
% 	In this case, $\Delta I^+(E^*)>(\frac{1}{e}-\epsilon) n^{10}$.
% 	%
% 	
% 	Suppose each cover of size at most $h$ covers at most $\frac{3}{4}$ of the elements. Then the probability of activate $G'$ is smaller than $e^{-n}$.
% 	%
% 	and $\Delta I^+(E') \le n^8 (n+m+1)+1+ e^{-n} n^{10}$.
% 	
% 	Clearly $\frac{\Delta I^+(E')}{\Delta I^+(E^*)} < \rho$.
% 	%
% 	Hence, a polynomial time $\rho$-approximation algorithm for influence maximization implies that we can distinguish between satisfiable instances of \textsc{Max-Cover} and instances in which at most $\frac{3}{4}$ of the elements are covered.
% 	\end{proof}

Now we can state the following result, whose proof is direct from the proof of the above theorem.

\begin{corollary}\label{cor:EAFixed}
	For any constant $\rho>0$, there is not any polynomial time algorithm returning a $\rho$-approximation to the \textsc{ECEA} when $B$ is finite even when there are two candidates and single-news-article messages, unless \Poly~$=$~\NP.
\end{corollary}

Thus, we focus on the case with unlimited budget. Since the maximum influence is reached when the network is fully connected, the optimal solution to \textsc{IMEA} with unlimited budget adds all the non-existing edges to the network and thus can be computed in polynomial time. An argument similar to the one used for edge removal shows that \textsc{ECEA} with unlimited budget, two candidates, and single-news-article messages is easy.
In particular, if the message is positive for $c_0$, \emph{i.e.}, $q_0=1$ or $q_1=-1$, we aim at maximizing the diffusion of the message and we add all the edges. If the message is negative for $c_0$,  \emph{i.e.}, $q_0=-1$ or $q_1=1$, we aim at minimizing the diffusion and we do not remove any edge.
From the previous arguments, we can directly state the following.
\begin{obs}
\label{obs:polyECEA}
	There exists a polynomial-time algorithm for the \textsc{ECEA} problem with single-news-article messages, two candidates and unlimited budget.
\end{obs}

 Next, we prove that increasing the number of candidates or allowing arbitrary messages makes the problem hard. 

\begin{theorem}\label{thm:EA_single_message}
	For any $\rho > 0$, there is not any polynomial time algorithm returning a $\rho$-approximation to the \textsc{ECEA} with single-news-article messages even when there are three candidates and the budget is unlimited, unless \Poly~$=$~\NP.
\end{theorem}
% \begin{proof}[Proof Sketch]
% 	Given an instance of \textsc{Set-Cover}, we build an instance of election control as depicted in Figure~\ref{fig:EA_single_message}.
% 	\begin{figure}[htb]
% 	\includegraphics[width=1\linewidth]{images/EA_single_message.pdf}
% 	\caption{Graph used in the reduction of Theorem \ref{thm:EA_single_message}.}
% 	\label{fig:EA_single_message}
% \end{figure}
% 	%
% 	We then prove that there exists a set $E^* \subseteq E$ with $\Delta_{\MoV}^+(E^*)>0$ if and only if \textsc{Set-Cover} is satisfiable.
% 	
% 	\textbf{If}.
% 	Given a set cover $X^*$, define as $E^*$ the set of the edge from $v_1$ to $L_1$ and all the edges from $L_1$ to $v_{x_i}, x_i \in X^*$.
% 	%
% 	If we add edges $E^*$, $c_0$ loses $nm-h-1+h$ votes, while $c_1$ loses $nm$ votes, and $\Delta_{\MoV}^+(E^*)=1$.
% 	
% 	\textbf{Only if}.
% 	The existence of a set $E^*$ with $\Delta_\MoV^+(E^*) > 0$ implies that the edge from $v_1$ to $L_1$ is added and $c_0$ loses at least $nm-h-1$ votes.
% 	%
% 	Thus, $c_1$ must lose all the not isolated nodes.
% 	%
% 	Since $\Delta_\MoV^+(E^*) > 0$, there are at most $h$ edges from $L_1$ to $v_{x_i}$ in $E^*$.
% 	%
% 	Hence, there are $h$ nodes $v_{x_i}$ that cover all the elements $z_i$ and \textsc{Set-Cover} is satisfiable.
% \end{proof}

\begin{proof}
	Given an instance of \textsc{Set-Cover}, we build an instance of election control as follows.
	We add a node $v_1$ with preferences $ \langle 0,1,2\rangle $ and seed it with $q_2=1$.
	We add a line $L_1$ of $ng-h-1$ nodes with preferences $ \langle 2,0,1\rangle $.
	We add a node $v_{x_i}$ for each set $x_i \in X$ with preferences $ \langle 2,0,1\rangle $.
	For each element $z_i \in N$, we add a line $L_{z_i}$ of $g$ nodes with ranking $ \langle 0,2,1\rangle $ and an edge from each $x_j \ni z_i$ to the first node of $L_{z_i}$.
	Moreover, we add $n^2g^2$ isolated nodes with preferences $ \langle 2,1,0\rangle $ and $n^2g^2$ isolated nodes with preferences $ \langle 1,2,0\rangle $.
	An example of network produced with the above mapping is depicted in Figure~\ref{fig:EA_single_message}.
	The only edges that can be added are the node from $v_1$ to $L_1$ and the edges from $L_1$ to each nodes $v_{x_i}$, \emph{i.e.}, these edges have probability $1$ and all other non existing edges have probability $0$.
	Notice that if no edge is added, all nodes will not change their votes,implying $\Delta_\MoV(\emptyset)=0$.
	We prove that there exists a set $E^* \subseteq E$ with $\Delta_{\MoV}^+(E^*)>0$ if and only if \textsc{Set-Cover} is satisfiable. 
	% that a constant factor for election manipulation would imply a constant factor approximation for $Set-Cover$.
	%
	\begin{figure}[htb]
		\includegraphics[width=1\linewidth]{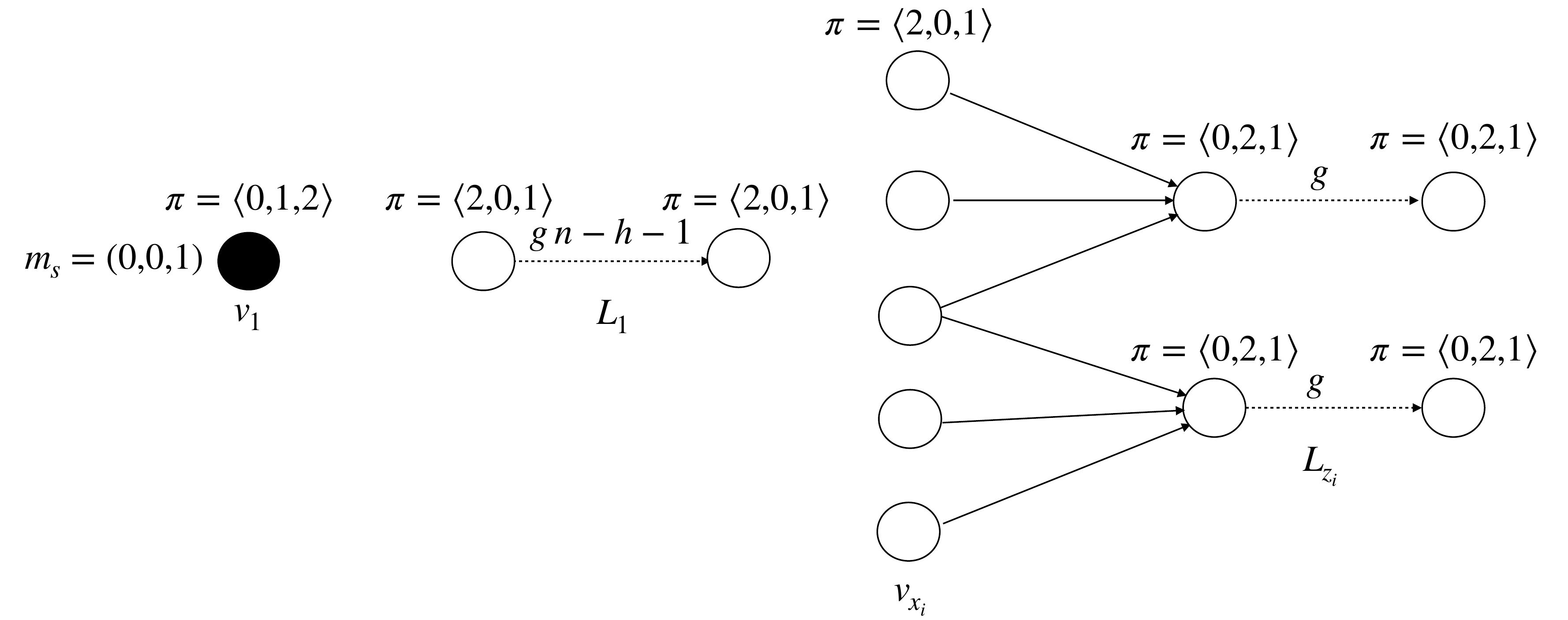}
		\caption{Structure of the election control problem used in the proof of Theorem \ref{thm:EA_single_message}.}
		\label{fig:EA_single_message}
	\end{figure}
	
	\textbf{If}.
	Given a set cover $X^*$, define as $E^*$ the set of the edge from $v_1$ to $L_1$ and all the edges from $L_1$ to $v_{x_i}, x_i \in X^*$.
	If we add edges $E^*$, $c_0$ loses $ng-h-1+h$ votes, while $c_1$ loses $n\,g$ votes, and thus $\Delta_{\MoV}^+(E^*)=1$.
	
	\textbf{Only if}.
	The existence of a set $E^*$ with $\Delta_\MoV^+(E^*) > 0$ implies that the edge from $v_1$ to $L_1$ is added and $c_0$ looses at least $n\,g-h-1$ votes.
	Thus, $c_1$ must lose at least $n\,g-h$ votes, implying that she loses all the not isolated nodes.
	Since $\Delta_\MoV^+(E^*) > 0$, $c_0$ can lose at most $h$ nodes $v_{x_i}$, \emph{i.e.}, there are at most $h$ edges from $L_1$ to $v_{x_i}$ in $E^*$.
	Hence, there are $h$ nodes $v_{x_i}$ that cover all the elements $z_i$ and \textsc{Set-Cover} is satisfiable.
\end{proof}

\begin{theorem}\label{thm:EA_multiple_messages}
	For any $\rho > 0$, there is not any polynomial time algorithm returning a $\rho$-approximation to the \textsc{ECEA} even with two candidates and unlimited budget, unless \Poly~$ =$~\NP.
\end{theorem}
% \begin{proof}[Proof Sketch]
% 	Consider an instance of \textsc{Set-Cover}. We suppose, w.l.o.g., $n>m$ and build a graph as depicted in Figure~\ref{fig:EA_multiple_messages}.
% 	\begin{figure}[htb]
% 	\includegraphics[width=1\linewidth]{images/EA_multiple_messages.pdf}
% 	\caption{Graph used in the reduction of Theorem \ref{thm:EA_multiple_messages}.}
% 	\label{fig:EA_multiple_messages}
% \end{figure}
% 	We prove that there exists a set $E^* \subseteq E$ with $\Delta_\MoV^+(E^*)>0$ if and only if \textsc{Set-Cover} is satisfiable.
% 	
% 	\textbf{If}.
% 	The set of added edges $E^*$ is composed by the edge between $v_2$ and $L_1$ and the incoming edge of each $v_{x_i}$ with $x_i \in X^*$.
% 	Thus, we have that $\Delta_\MoV^+( E')=1$.
% 	
% 	\textbf{Only if}.
% 	Suppose we do not add the edge between $v_2$ and $L_1$. In this case, \MoV\ do not change and $\Delta_\MoV^+=0$.
% 	%
% 	Thus, the edge between $v_2$ and $L_1$ must belong the set of added edges $E^*$.
% 	%
% 	We have two possible live graphs, $H_1$ if the edge between $v_1$ and $v_2$ is active, $H_2$ otherwise.
% 	%
% 	In $H_1$, all line $L_z$ must be active and $\Delta_\MoV^+(E^*,H_1) = 2n^2$.
% 	%
% 	In $H_2$ at most $h$ edge from $L_1$ to voters $v_{x_i}$ can be added. Thus, there is a set cover of size $h$.
% \end{proof}

\begin{proof}
	Consider an instance of \textsc{Set-Cover}. We suppose, w.l.o.g., $n>g$ and build a graph as follow. 
	We add a node $v_1$ with preferences $ \langle 1,0\rangle $ and seeded with messages $q_0=1$ and $q_1=-1$, and a node $v_2$ with preferences $ \langle 1,0\rangle $ and seeded with message $q_1=1$. Moreover we add an edge with probability $\frac{1}{2}$ between $v_1$ and $v_2$.
	We add a line $L_1$ of $n^2-h-1$ nodes with preferences $ \langle 1,0\rangle $.
	% and an edge from $v_3$ to the first node of the line.
	%
	We add a node $v_{x_i}$ for each set $x_i \in X$ with preferences $ \langle 1,0\rangle $.
	For each element $z_i \in N$, we add a line $L_{z_i}$ of $n$ nodes with preferences $ \langle 0,1\rangle $ and an edge from each $x_{j} \ni z_i$ to the first node of $L_{z_i}$.
	Moreover, we add $g-h+1$ isolated nodes with preferences $ \langle 0,1\rangle $.
	The edges that can be added are: the edge from $v_2$ to the first node of $L_1$ with probability $1$ and edges from the last node of $L_1$ to all $v_{x_i}$ with probability $1$ (all other non-existing edges have probability $0$).

	\begin{figure}[!ht]
		\includegraphics[width=1\linewidth]{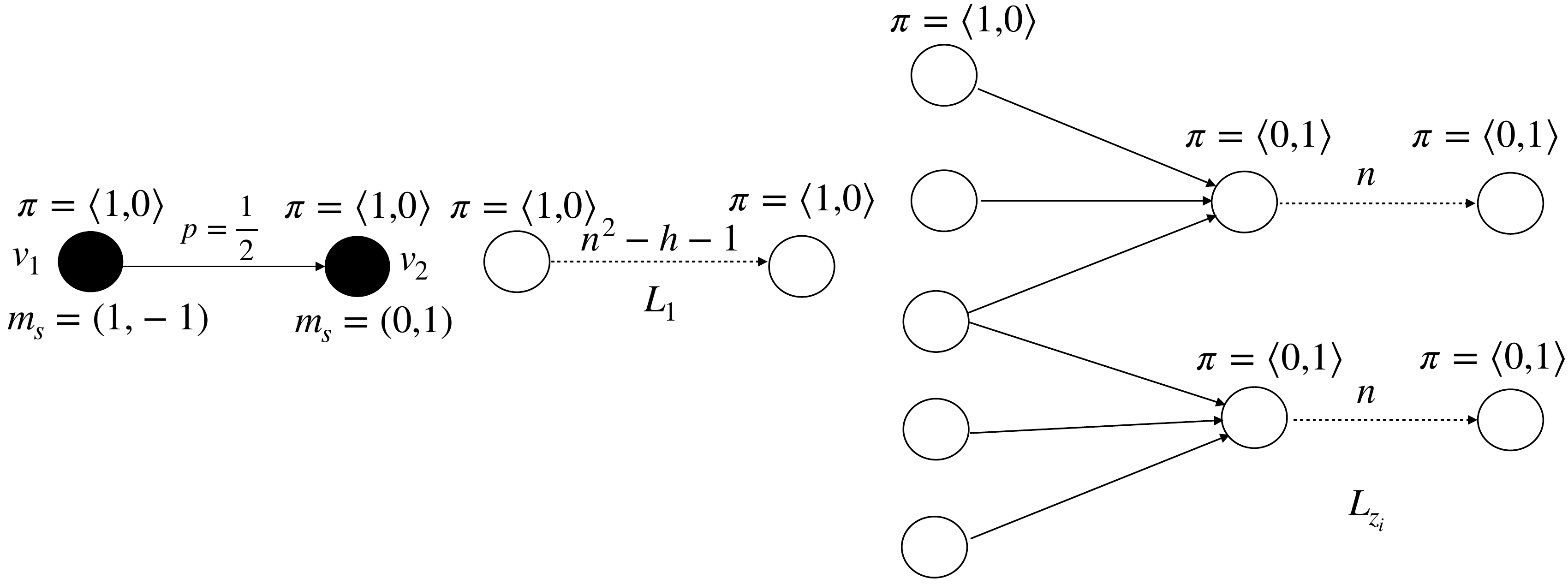}
		\caption{Structure of the election control problem used in the proof of  Theorem \ref{thm:EA_multiple_messages}.}
		\label{fig:EA_multiple_messages}
	\end{figure}
	
	%An example of network produced with the above mapping is depicted in Figure~\ref{fig:EA_multiple_messages}.
	%
	Notice that, if no edges are added, no voter changes her votes and $\MoV(\emptyset)$ is $0$.
	We prove that there exists a set $E^* \subseteq E$ with $\Delta_\MoV^+(E^*)>0$ if and only if \textsc{Set-Cover} is satisfiable.

	\textbf{If}.
	The set of added edges $E^*$ is composed by the edge between $v_2$ and $L_1$ and the incoming edge of each $v_{x_i}$ with $x_i \in X^*$.
	We have two possible live graphs: $H_1$ if the edge between $v_1$ and $v_2$ is active, $H_2$ otherwise.
	$\Delta_\MoV^+(E', H_1)=2n^2$ and $\Delta_\MoV^+(E', H_2)=2 (-n^2+h+1-h)=-2n^2+2$.
	Thus, $\Delta_\MoV^+( E')=1$.
	
	\textbf{Only if}.
	Suppose we do not add the edge between $v_2$ and $L_1$. In this case, \MoV\ does not change and $\Delta_\MoV^+=0$.
	Thus, the edge between $v_2$ and $L_1$ must belong the set of added edges $E^*$.
	We have two possible live graphs: $H_1$ if the edge between $v_1$ and $v_2$ is active, $H_2$ otherwise.
	Assume by contradiction that in $H_1$ not all lines $L_{z_i}$ vote for $c_0$.
	This implies that $\Delta_\MoV^+(E^*) \le \frac{2[n(n-1)]-2[n^2+h+1]}{2}<0$.
	Hence, in $H_1$, all line $L_z$ must be active and $\Delta_\MoV^+(E^*,H_1) = 2n^2$.
	In $H_2$, $\Delta_\MoV^+(E^*)$ must be larger than $-2n^2+1$ and at most $h$ nodes $v_{z_i}$ can be active, \emph{i.e.}, at most $h$ edges from $L_1$ to voters $v_{x_i}$ can be added. Thus, there exists a set cover of size $h$, leading to a contradiction.
\end{proof}

\section{Reoptimization Complexity}
\label{sec:reoptimization}
In this section, we show a form of robustness of our hardness results. Specifically, we consider the following reoptimization setting.
\begin{definition}
	An election control through social influence by seeding reoptimization problem $ReOpt(I,S^*,e,o)$  is defined as follows.
	\begin{itemize}
		\item{INPUT: $(I,S^*,e,o)$, where $I$ is an instance of election control, $S^*$ is an optimal solution to $I$, $e \in  V \bigtimes V$ is an edge and $o \in [0,1]$ is a probability.
		}
		\item{OUTPUT: the optimal solution to $I_1$, where $I_1$ is obtained changing the probability of edge $e$ to $o$ in the instance $I$.
		}
	\end{itemize}
\end{definition}

We prove that Theorem \ref{thm:inapprox} can be extended to prove the hardness of reoptimization.

\begin{theorem}
	For any $\rho > 0$ even depending on the size of the problem, there is not any poly-time algorithm returning a $\rho$-approximation to the reoptimization problem for ECS, unless $\mathsf{P} = \mathsf{NP}$.
\end{theorem}

\begin{proof}%[Proof Sketch]
	Consider the reduction in Theorem \ref{thm:inapprox}. We build an instance $I$ of ECS in which we replace the set cover instance in $G_1$ with the following graph. For each $z_i \in Z$, there is a node $v_{z_i}$. For each $x_i \in X$, we add two nodes $v_{x_{i,1}}$ and $v_{x_{i,2}}$, and an edge from $v_{x_{i,1}}$ to $v_{x_{i,2}}$. Moreover, we add an edge from $v_{x_{i,2}}$ to all $v_{z}$, $z \in x_i$. Finally, we add a node $v^*$ with an edge from $v^*$ to all nodes $v_z$, $z \in Z$, and an edge from $v^*$ to $v_{x_{1,2}}$ (or any node $x_{i,2}$). We modify the graphs $G_2$ and $G_3$ in such a way that $c_1$ needs the votes of $2h+n$ nodes of $G_1$.
	Let $S^*$ be the optimal solution of $I$ that includes seeds $v_z$, any $h$ nodes $v_{x_{i,1}}$ and a node in $G_2$. 
	Consider the problem $ReOpt(I,S^*,(v^*,v_{x_1}),0)$, its optimal solution is the optimal solution of the optimization problem over $I$. If $ReOpt(I,S^*,(v^*,v_{x_1}),0)$ can be approximated in polynomial time, then set cover can be solved in polynomial time.
\end{proof}

Similarly, we consider the reoptimization problems for edge removal or edge addition.

\begin{definition}
	An election control trough social influence by edge removal or addition reoptimization problem $ReOpt(I,E^*,e,o)$  is defined as follows.
	\begin{itemize}
		\item{INPUT: $(I,E^*,e,o)$, where $I$ is an instance of election control, $E^*$ is an optimal solution to $I$, $e \in  V \bigtimes V$ is an edge and $o \in [0,1]$ is a probability.
		}
		\item{OUTPUT: the optimal solution to $I_1$, where $I_1$ is obtained changing the probability of edge $e$ to $o$ in the instance $I$.
	}
\end{itemize}
\end{definition}
The following theorem shows a general result, that extends the hardness of the optimization problem to its reoptimization variant whenever a simple condition is satisfied.
\begin{theorem}\label{thm:reopt}
	For the set of election control problems by edge removal or addition with $\max_v \{\max_{c_i} \pi_v(i) - \min_{c_i} \pi_v(i)\}= O(poly(size(I)))$, reoptimization is as hard as optimization.
\end{theorem}
\begin{proof}
	Consider an instance $I$ of election control with $G=(V,E,p)$.
	By assumption $d=\max_v [\max_{c_i} \pi_v(i) - \min_{c_i} \pi_v(i)]= O(poly(size(I)))$, \emph{i.e.}, $d$ is polynomially upper bounded in the instance size.
	We build a graph $G_1$ with $d+1$ nodes $\{v_{i}\}, i \in \{0,\dots,d\}$ with seeds $q_0=1$.
	We add a node $v^*_1$ with an edge from each node in $v_i$ to $v^*_1$.
	We add a node $v^*_2$ with an edge from $v_1^*$ to $v_2^*$.
	Moreover, we add an edge from $v_2^*$ to any node of $G$.
	In edge addition instances, we set $p=0$ for all (non-existing) edges among $v_i$ and $G$.
	Finally, we set the preferences of $v_0$ and $v_i$ s.t.~they will vote for $c_0$, \emph{i.e.}, $\pi(0)>\pi(i)$ holds for every $c_i \neq c_0$.
	
	Notice that, since all nodes in $G$ receive $d+1$ positive messages on $c_0$ and $c_0$ is loosing by at most $d$ in each preferences, all nodes will vote for $c_0$.
	Thus the optimal solution removes/adds no edges.
	Consider the problem $ReOpt(I,\emptyset,(v_1^*,v_2^*),0)$, its optimal solution is the optimal solution of the optimization problem over $I$.
\end{proof}
In the reductions used in the proofs of all the theorems provided in the previous sections, $\max_v \{\max_{c_i} \pi_v(i) - \min_{c_i} \pi_v(i)\}$ is constant. Hence, as a corollary of Theorem~\ref{thm:reopt}, we have that all our hardness results on optimization problems extend to their reoptimization variants.

\section{Conclusions and Future Work}
\label{sec:conclusions}
In this work, we analyze the problem of manipulating the result of an election (a.k.a. election control trough social influence) by some forms of manipulations. More precisely, we investigate both the case in which the manipulator can make seeding and the case in which the manipulator can alter the network by removing or adding edges. We prove a tight characterization of the settings in which computing an approximation to the best manipulation can be infeasible or feasible. In particular, our results show that, except for trivial classes of instances, the manipulation is hard, even when one accepts an approximation of the margin of victory. In particular, we provide a comprehensive study, investigating various subcases to identify the minimal subsets of instances for which the election control problem is hard. We show that manipulation by seeding is hard even with basic instances such as line graphs and that the most known greedy algorithms for social influence do not provide any approximation factor even with basic graphs. In the case of edge removal or addition, we also show that, even when the manipulator has an unlimited budget, the problem is hard. Interestingly, we derive a similar result also to influence maximization/minimization, as this problem was unexplored so far. Finally, we show that our hardness results hold in a reoptimization variant.  This bundle of results is, therefore, positive for the democracy, as it is unlikely that a manipulator can effectively manipulate an election.

While we provided a poly-time constant-approximation algorithm in many settings, we did not try to optimize the approximation ratio. Hence, it would be interesting to design algorithms that can  improve on ours, and close the small gaps existing among our results. Furthermore, it would be interesting to analyze other generalizations of our model, \emph{e.g.}, different models for information diffusion and time-evolving networks. 
% \textcolor{orange}{***nel caso di seeding, si potrebbe cercare di capire quale sia il numero minimo di votanti che non cambiano idea affinche' il problema sia hard***}

\bibliography{elections}

\appendix
\section{Omitted Proofs}
\label{sec:appendix}

\propositionone*

 \begin{proof}
 Consider the graph given in Figure~\ref{fig:example1}: it is composed of three graphs $G_1,G_2,G_3$, and both $G_2$ and $G_3$ are composed of two subgraphs, respectively, $G_{2,1},G_{2,2}$ and $G_{3,1},G_{3,2}$. Furthermore, the graph is undirected and each node has a degree of at most $2$. The influence probability associated with every edge is one.
 According to the preference ranks of the nodes, candidate $c_2$ collects 5 votes, while candidates $c_1$ and $c_0$ gather $7$ votes each. Thus, the actual margin of victory of $c_0$ is equal to zero. Suppose $B=2$.

\begin{figure}[!htb]
	\centering
	\includegraphics[width=0.8\columnwidth]{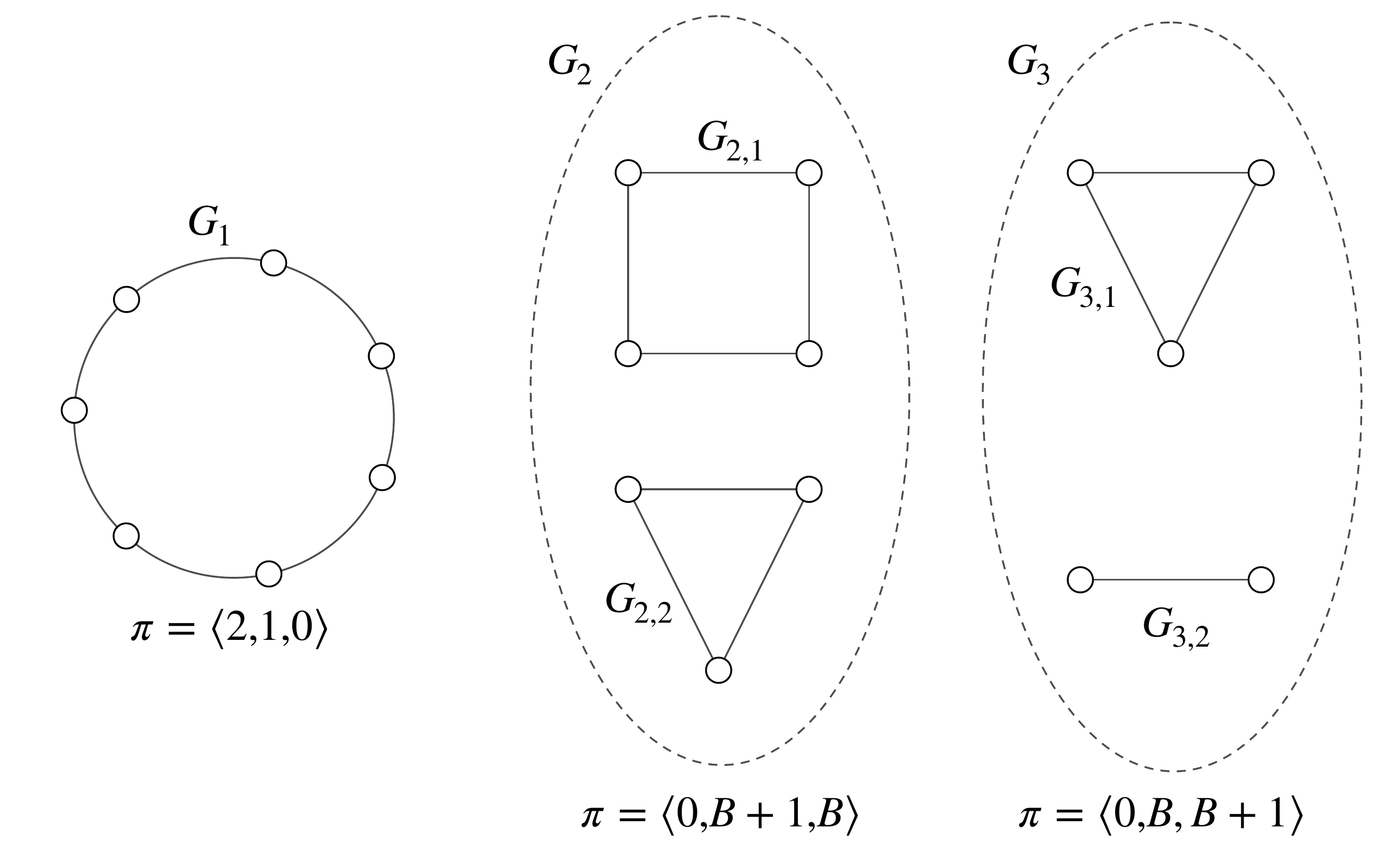}
	\caption{example of a small undirected network in which the greedy algorithm performs badly.}
	\label{fig:example1}
\end{figure}

Since, for all the nodes of $G_2$ and $G_3$, the score difference between $c_0$ and the most-preferred candidate is larger than the budget $B$ and $B=2$, it is clear that $c_0$ cannot get any further vote. Then, to increase the margin of victory of $c_0$, it is necessary that $c_2$ obtains some of the $c_1$'s votes. The optimal solution $(S^*,M^*)$ is that, while $c_0$ keeps $7$ votes, $c_1$ and $c_2$ collect $6$ votes each, providing $\Expec_H \left[\MoV(S^*,M^*,E, H)\right] = 1$. This can be obtained by setting $m_v(2)=1$ for a single $v \in G_{2,2}$ and $m_v(1)=1$ for a single $v \in G_{3,2}$.

However, this solution cannot be found by any algorithm adopting the greedy approach described above. Indeed, we next show that $\mathcal{F}(\emptyset,())$ is empty, and thus the algorithm never adds any seed in $S$: clearly, $\mathcal{F}(\emptyset,())$ cannot contain any pair $(s,m_s)$ that increases the number of votes of $c_0$; moreover,
by seeding a node in $G_1$ the margin of victory clearly cannot increase (it either remains unchanged, or it decreases if $c_0$ ceases to be the best ranked candidate); similarly, by seeding a node in $G_2$, either the margin of victory reduces (if $c_2$ passes $c_1$) or remains unchanged; finally, by seeding one node in $G_3$ either the margin of victory reduces (if $c_1$ passes $c_2$) or remains unchanged.

Hence, the greedy solution results in a zero margin of victory, and thus it cannot be a $\rho$-approximation.
\end{proof}

\propositiontwo*

 \begin{proof}
Consider the graph in Figure~\ref{fig:example2}. It is composed of 5 subgraphs $A_1,A_2,A_3,A_4,A_5$, where $A_1$ is a directed line of $7\, r$ nodes, while $A_2,A_3,A_4,A_5$ are directed trees in which there is a root and the remaining nodes are children of the root. The specific number of nodes of every subgraph is reported in the figure. The influence probability associated with every directed edge is one. Let $r > \frac{2}{\rho}$.  Observe that $|V| = 19\,r$ and therefore $\rho > \frac{38}{|V|}$.
  According to the preference ranks of the nodes, the candidate $c_2$ collects $5\,r$ votes, while candidates $c_1$ and $c_0$ gather $7\,r$ votes each. Then, the actual margin of victory $c_0$ is equal to zero. Suppose $B=2$.

 \begin{figure}[!htb]
	\centering
	\includegraphics[width=0.7\columnwidth]{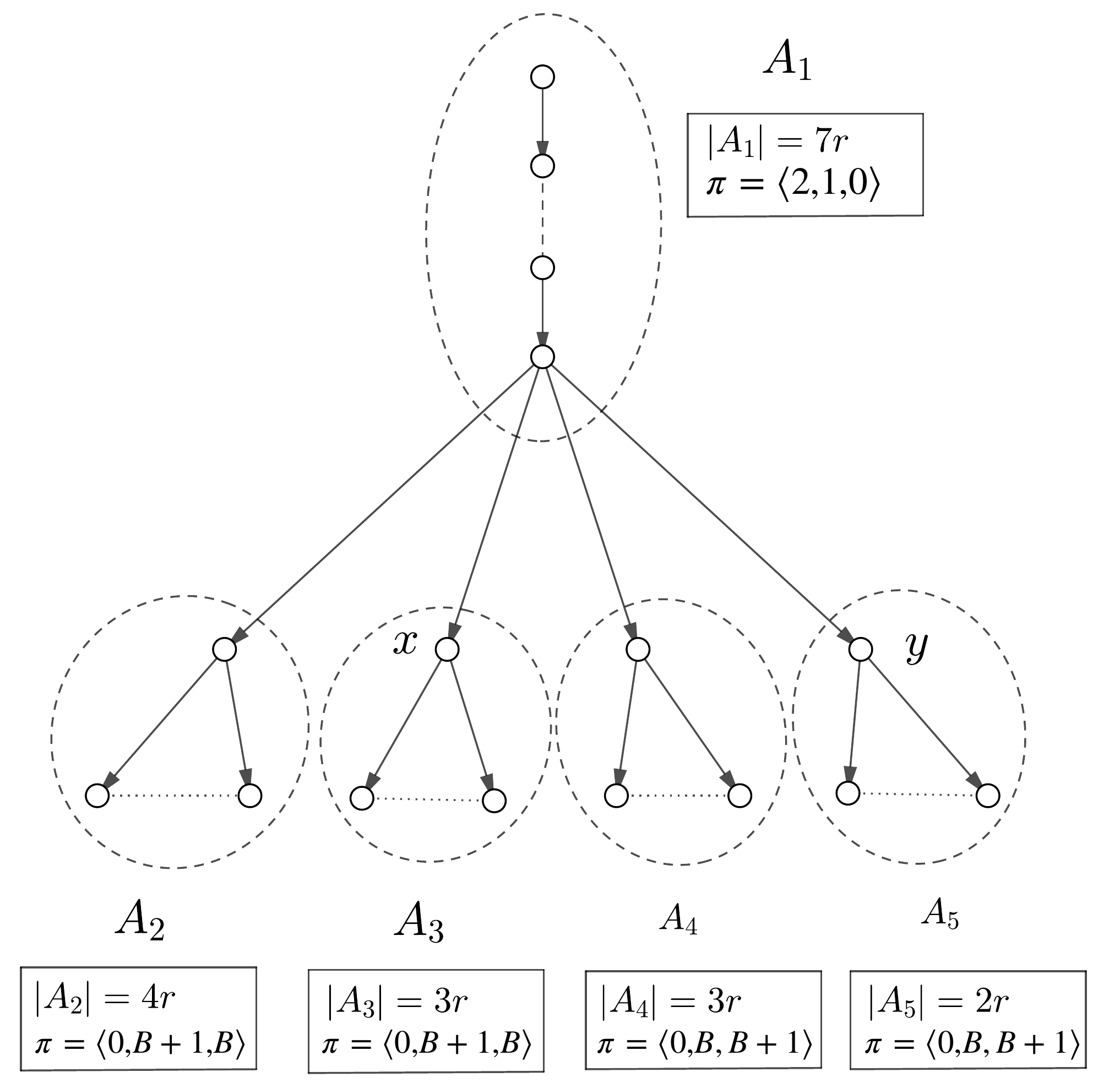}
	\caption{Example of a tree in which the greedy algorithm performs badly.}
	\label{fig:example2}
 \end{figure}

Due to the budget constraint, $c_0$ cannot get any further vote, and thus, to increase the margin of victory of $c_0$, we need that $c_2$ gets some of the $c_1$'s votes. The optimal solution $(S^*,M^*)$ is then obtained by setting that $S^*=\{x, y\}$ (see in the figure which nodes are labeled as $x$ and $y$) has $m_x$ with $m_x(2)=1$ and $m_y$ with $m_y(1)=1$, and the expected margin of victory is $r$.
 However, this solution cannot be found by any algorithm adopting the greedy approach. Indeed, we have that no pair $(s,m_s)$ can increase the number of votes of $c_0$; moreover, by seeding a node in $A_1$ the margin of victory clearly cannot increase (it either remains unchanged, or it decreases if $c_0$ ceases to be the best ranked candidate); by seeding a node in $A_4$ or in $A_5$, either the margin of victory decreases (if $c_2$ passes $c_1$) or remains unchanged; finally, by seeding the root of $A_1$ or the root of $A_2$, either the margin of victory decreases (if $c_2$ passes $c_1$) or remains unchanged. Hence, the only action that the greedy algorithm can take would be to select as seed either a leaf of $A_1$, or a leaf of $A_2$ and letting them to change its vote from $c_1$ to $c_2$. By repeating the argument, we have that the two seeds selected by a greedy algorithm must be two leaves from $A_1 \cup A_2$. So, the expected margin of victory is $2$, and the approximation factor is $\frac{2}{r} < \rho$.
 \end{proof}

\theoremtwo*
 \begin{proof}
	
	We reduce from \textsc{Partition}. This problem, given a set of positive integers $A = \left\lbrace a_1, a_2, ... , a_n \right\rbrace $, asks if there is a subset $K \subset A$ whose sum is equal to $t$, where $t= \sum_{a \in A}{a}/2$. It easy to see that the \NPHard ness of this problem holds even imposing that the cardinality of $S$ is $k\le n/2$.
	
	\begin{figure}
		\centering
		\includegraphics[width=0.7\linewidth]{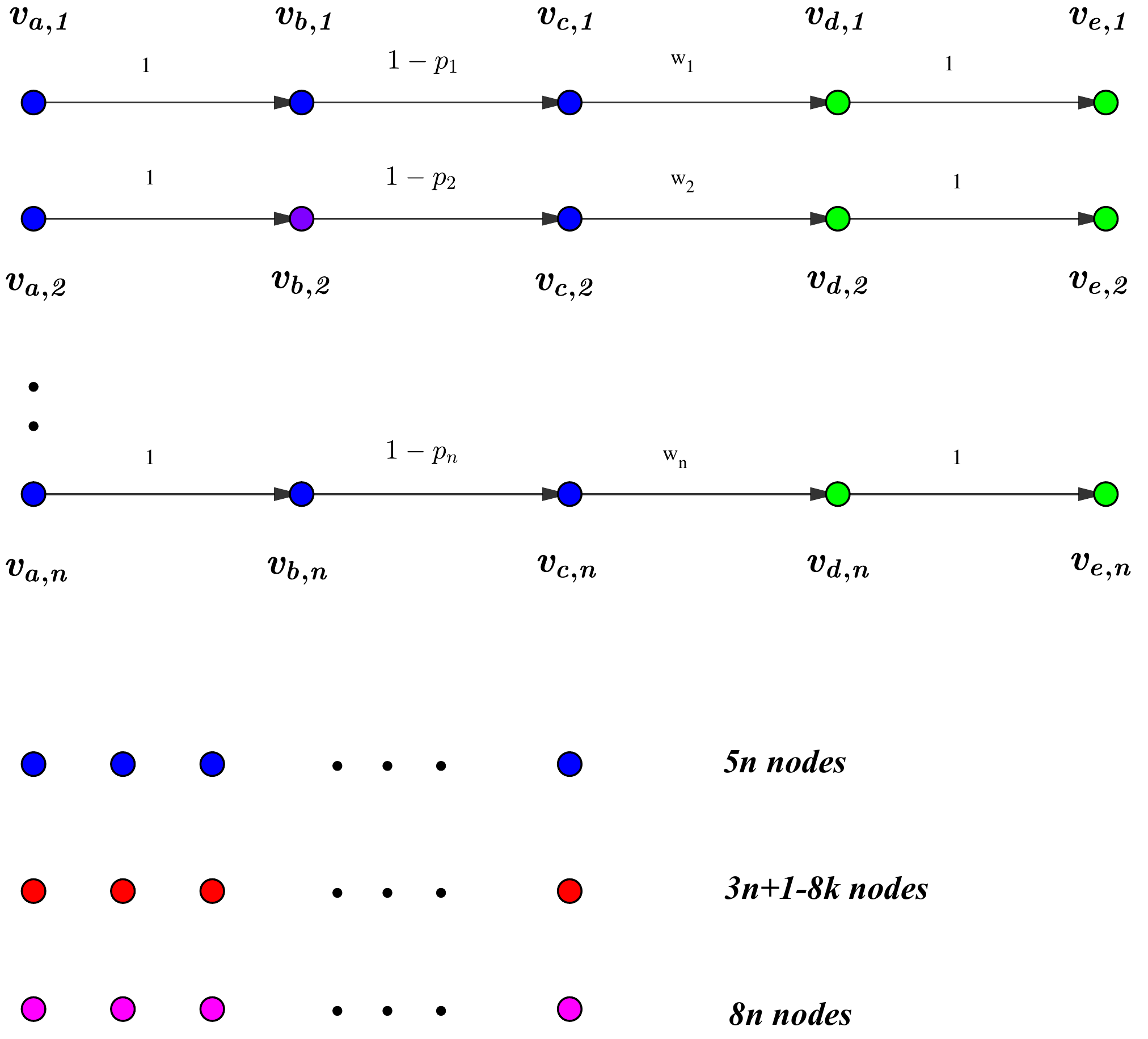}
		\caption{Structure of the election control problem used in the proof of  Theorem \ref{thm:line}. Blue nodes represent voters with preference rank $\langle 0,B+2,B+1,1\rangle$, green nodes represent voters with preference rank $\langle 0,B,B+1,B+2\rangle$, while red nodes and pink nodes have preference rank, respectively, $\langle 0,B+1,B+2,B\rangle$, and $\langle 4,3,2,1\rangle$.}
		\label{fig:linea}
	\end{figure} 
	
	Hence, given a set $A$ of positive integers, a target $t$ and a number $k$ as the input of \textsc{Partition}, we build a graph as in Figure \ref{fig:linea}. There are four candidates $c_0, c_1, c_2, c_3$ and budget $B=k$. Some voters are isolated nodes, while the others are arranged in $n$ independent lines.~\footnote{The graph can be seen as a single line with zero probability links going from $v_{e,j}$ to $v_{a,j+1}$, for $j = 1,2,\ldots,n-1$, and connecting the isolated nodes.} There are $5\,n$ isolated nodes with preference rank $\langle 0,B+2,B+1,B\rangle$, $8\,n+1-8\,k$ isolated nodes with preference rank $\langle 0,B+1,B+2,B\rangle$ and $3\,n$ nodes with preference rank $\langle 4,3,2,1\rangle$. Each line is composed by 5 nodes. For each line $i \in \{1, 2, \dots , n\}$, the first three nodes have preference rank $\langle 0,B+2,B+1,1\rangle$, the last two nodes have preference rank $\langle 0,B,B+1,B+2\rangle$, the edges connecting $v_{a,i}$ to $v_{b,i}$ and $v_{d,i}$ to $v_{e,i}$ have probability $1$, and the edges $(v_{b,i},v_{c,i})$ and $(v_{c,i},v_{d,i})$ are activated, respectively, with probability $1 - p_i$ and $w_i$, where:
	\begin{center}
		$ \displaystyle
		p_i = \frac{a_{i}}{4t} \: \:, \quad
		w_i = \frac{2^{-4p_i}}{(1-p_i)\, (2 \ln2)^{1/k}}.
		$
	\end{center}

	Note that, since $a_i < t, \forall i \in \{ 1, \dots, n\}$, then $p_i < \frac{1}{4}, \forall i$. It is easy to see that, since $p_i < \frac{1}{4}$, it holds $w_i < 1, \forall i$.
	$c_1$ collects $8\,n $ votes, $c_2$ collects $8\,n + 1 - 8\,k $ votes, $c_3$ collects $2\,n$ votes, and $c_0$ collects $8\,n$ votes. Hence, $\MoV(\emptyset,(),E,H)=0$. 
	
	Since $c_0$ cannot gain votes, the only way to increase the margin of victory is making $c_1$ loose some of her votes in favor of the other candidates. 
	First, we prove that the optimal solution is given by a set of $k$ nodes $v_{a,i}$ with messages $m_s(2)=1$ (or equivalently $m_s(1)=-1$).
	Consider any seed set $S'$ of $k$ nodes $v_{a,i}$ and let $K'=\{i:v_{a,i}\in S'\}$. Notice that $c_1$ has more votes than $c_2$ unless all the edges among $v_{b,i}$, $v_{c,1}$ and $v_{d,i}$ are active for all $i \in K'$. This happens with probability  $\prod_{i \in K'} (1 - p_i) w_i$.  In this case $c_1$ has $8\,n-3\,k$ votes and $c_2$ has $8\,n+1-3\,k$.
	When at least an edge is not active, the \MoV\ is given by the number of nodes lost by $c_1$.
	Hence, the expected \MoV \ is 
	
	\begin{align*}
	\Expec_H[\Delta_\MoV^S (S',M,H)]  & = 2 k + \prod_{i \in K'} (1- p_i) - \prod_{i \in K'} (1 - p_i)\,w_i \\ & = 3 k - \prod_{i \in K'} p_i - \prod_{i \in K'} (1 - p_i)\,w_i.
	\end{align*} 
	
	Suppose the optimal solution $S^*$ takes some seeds $S^*_1 \subseteq S^*$ not in the set $\{v_{a,i}\}$. Then, for each seed $s \in S_1^*$, the expected number of votes of $c_1$ decreases of at most $1+ (1-p_i)=2-p_i$.
	Define the set $R(z)$ as the set of the $z$ smallest $p_i$.
	Consider the set $R(k)$ and the set of seeds $S''=\{v_{a,i}\}_{i \in R(k)}$. $\Expec_H[\Delta_\MoV (S'',M,H)] =3 k - \prod_{i \in R(k)}  p_i - \prod_{i \in R(k)} (1 - p_i)\,w_i> 3k-1-\prod_{i \in R(k)}p_i$, where the last inequality follows from $\prod_{i \in R(k)} (1 - p_i)\,w_i < 1$.
	It follows that $\Expec_H[\Delta_\MoV^S (S^*,M,H)]\le |S^*_1|(2-p_i) + \sum_{i \in R(|S^*| -|S^*_1|)} (3- p_i)< \Expec_H[\Delta_\MoV^S (S'',M,H)$.
	
	We proved that in the optimal solution $S^*$, all the seeds are placed at the beginning of the line. Let $K^*=\{i: v_{a,i} \in S^* \}$ and let $x = \sum_{i \in K^*} p_i$. Then the derivative of $\Expec_H[\Delta_\MoV^S(S^*,M^*,H)]$ with respect to $x$ is equal to $- 1 + 2^{-4x+1}$. 
	This means that the value of $x$ that maximizes the margin of victory is $\frac{1}{4}$, which is equivalent to $\sum_{i \in K^*} a_i = t$. This holds if and only if we reduce from a "yes" instance of partition. Hence a polynomial time algorithm for the ECS problem, would allow us to solve \textsc{Partition} in polinomial time, leading to a contradiction unless \Poly=\NP.\end{proof}

\theoremthree*
\begin{proof}
	Given an indirect graph $G$, we add a voter $v_x$ for each vertex of $X$ with preference rank $\langle 1,0 \rangle$.
	For each edge $n$ of $N$, we add a node $v_{n}$ with preference rank $\langle 0,2 \rangle$.
	Finally, for each edge $n=(x,x')$, we add an edge (with probability $1$) from $v_x$ to $v_{n}$ and from $v_{x'}$ to $v_{n}$.
	The budget is $B$ and it is easy to see that, in the optimal solution, the manipulator sends only messages $(1,0)$ or $(0,-1)$.
	
	First, we prove that for every solution $S$ to the ECS problem, we can construct in polynomial time a solution with at least the same $\Delta_\MoV^S$ and all the seeds in the set $\{v_x\}$. Let $S$ be a solution to the ECS problem. Suppose $S$ includes some nodes in $v_n$, where $n=(x,x')$. If both or neither of $v_x$ and $v_{x'}$ are seeds, the seed in $v_n$ is useless and can be removed. Otherwise, only $v_x$ (or only $v_{x'}$) is a seed. In this case, we can replace the seed in $v_n$ with a seed in $v_{x'}$ (or $v_x$) without decreasing $\Delta_\MoV^S$.
	
	It can be observed that the value of $\Delta_\MoV^S$ for a seed set $S'$ that includes only seeds $v_{x}$ is equal to $2|d(X')|$, where $X'= \{x \in X, v_x \in S'\}$ and $d(X')$ is the set of edges connecting nodes in $X'$. In particular, the set of edges that change the vote in favor of $c_0$ is the set of nodes $v_n$, where $n=(x,x')$ such that both $v_x$ and $v_{x'}$ are seeds. This is exactly the number of edges in the subgraph $X'$. Since each new vote for $c_0$ decreases by one the votes for $c_1$, it follows that $\Delta_\MoV^S(S',M^*)=2\,d(X')$. 
	
	Let $A$ be an algorithm for a poly-time  $\rho$-approximation algorithm for ECS that return a solution $S'$. Construct the solution $S''$ that includes only nodes in $v_x$. Let $X'=\{x \in X, v_x \in S''\}$ be the solution to \textsc{DkS} obtained from $S''$. $\frac{d(X')}{d(X^*)}= \frac{\Delta_{\MoV}^S(S'',M'')/2}{\Delta_{\MoV}^S(S^*,M^*)/2} \ge \rho$. Thus from $A$, we can construct a polynomial time $\rho$-approximation algorithm for \textsc{DkS}.\end{proof}

% \newpage
% \appendix
% \section{Notation}
% \input{notation}

\end{document}